\documentclass{article}

\usepackage[preprint]{neurips_2026}

\usepackage[utf8]{inputenc}
\usepackage[T1]{fontenc}
\usepackage{amsmath,amssymb,amsthm,mathtools}
\usepackage{bm}
\usepackage{enumitem}
\usepackage{graphicx}
\usepackage{booktabs}
\usepackage{microtype}
\usepackage{placeins}
\usepackage{url}
\usepackage{hyperref}
\hypersetup{
  hidelinks,
  pdftitle={Exact Attention Sensitivity and the Geometry of Transformer Stability},
  pdfauthor={Seyed Morteza Emadi}
}

\graphicspath{{figures/}}

\newtheorem{theorem}{Theorem}[section]
\newtheorem{lemma}[theorem]{Lemma}
\newtheorem{corollary}[theorem]{Corollary}
\newtheorem{definition}[theorem]{Definition}
\newtheorem{remark}[theorem]{Remark}

\newcommand{\R}{\mathbb{R}}
\newcommand{\thetacert}{\theta_{\mathrm{cert}}}
\DeclareMathOperator{\softmax}{softmax}
\DeclareMathOperator{\Diag}{Diag}
\DeclareMathOperator{\MHA}{MHA}
\DeclareMathOperator{\LN}{LN}
\DeclareMathOperator{\FFN}{FFN}

\title{Exact Attention Sensitivity and the Geometry of Transformer Stability}

\author{%
  Seyed Morteza Emadi \\
  UNC-Chapel Hill, Kenan-Flagler Business School \\
  Chapel Hill, NC, USA \\
  \texttt{seyed\_emadi@kenan-flagler.unc.edu} \\
}

\begin{document}

\maketitle

\begin{abstract}
We develop a sensitivity analysis for transformer attention in a geometry aligned with tokenwise computation. Our main result is the exact identity
$\|J_\tau(u)\|_{\infty\to1}=\theta(p)/\tau$ for the Jacobian $J_\tau(u)$ of the tempered softmax $u\mapsto\softmax(u/\tau)$, where $\theta(p)=4\max_{S\subseteq[L]}p(S)(1-p(S))$ measures how evenly the attention distribution can be bisected rather than how concentrated it is. We combine this identity with a block-$\infty$/RMS norm under which row-stochastic attention mixing is nonexpansive. This yields a distribution-aware local Jacobian bound for multi-head attention and a sequence-length-independent Lipschitz bound on bounded input sets, with explicit dependence on width, input magnitude, temperature, and projection norms. We also identify a structural distinction between normalization placements: a pre-LN residual-sublayer Jacobian contains an additive identity term, whereas a post-LN residual-sublayer Jacobian does not. A LayerNorm projection lemma gives a sufficient condition under which the LayerNorm-only term in the post-LN expansion contracts geometrically; the condition is not tested by our experiments. Across three Pre-LN early-training runs of $774$M-parameter models, attention becomes substantially more concentrated while the median lower-bound certificate for $\theta(p)$ remains near one at every sampled layer and checkpoint. This certifies near-maximal exact sensitivity for at least half of the sampled rows within each layer. A minority of rows enters a dominant-atom regime with lower exact sensitivity, consistent with the deterministic relationship between $p_{\max}$ and $\theta(p)$.
\end{abstract}

\section{Introduction}
\label{sec:intro}

Transformers \citep{vaswani2017attention,touvron2023llama} can be difficult to train reliably at scale: optimization may exhibit loss spikes, gradient explosions, or divergence. Practitioners mitigate these failures through normalization placement \citep{xiong2020layer}, residual and initialization scaling \citep{wang2022deepnet}, and learning-rate schedules \citep{goyal2017accurate}, yet the mechanisms behind these choices remain incompletely understood. Here we focus on three connected questions. What is the exact sensitivity of softmax in a perturbation geometry matched to attention? How does this local sensitivity propagate through a complete attention block? And what structural difference does normalization placement create in residual-sublayer Jacobians?

Several strands of theory address complementary parts of this problem. \emph{Global worst-case analyses} derive Lipschitz bounds in $\ell_2$ or Frobenius geometry \citep{kim2021lipschitz,dasoulas2021lipschitznorm,qi2023lipsformer}; under token-aggregated norms, these bounds can carry explicit sequence-length dependence. \emph{Local and distribution-aware analyses} \citep{yudin2025localLipschitz,castin2024smooth} refine these bounds using properties of the attention distribution or of a restricted input domain, but do not directly provide complete-block bounds that preserve the distinction between pre-LN and post-LN architectures. Generalization-theoretic work \citep{edelman2022inductive} uses a max-row norm inside covering-number arguments and combines $W^Q$ and $W^K$ into a single spectral factor. Our analysis instead retains $W^Q,W^K,W^V,W^O$ separately, which allows the value and attention-weight sensitivity pathways to be distinguished.

\paragraph{Our approach: architecture-aligned geometry.}
We develop a stability framework that respects the tokenwise structure of transformer computation. The key modeling choice is geometric: LayerNorm acts on each token independently, while attention outputs row-stochastic mixtures of value vectors. This motivates a norm that treats tokens separately and aggregates across tokens by a maximum rather than a Frobenius sum. We therefore work in a block-$\infty$/RMS geometry. In this geometry, attention mixing is automatically nonexpansive, LayerNorm has a bounded Lipschitz constant independent of input scale, and the resulting layerwise bounds contain no explicit sequence-length factor. From this geometric baseline, we derive an exact identity for the softmax Jacobian in $\ell_\infty\to\ell_1$ geometry and then propagate it through complete pre-LN and post-LN layers.

\paragraph{Contributions.}
\textbf{(1) Exact softmax sensitivity.} We derive the exact cross-norm identity $\|J_\tau(u)\|_{\infty\to 1}=\theta(p)/\tau$ for the tempered-softmax Jacobian, where the balanced-mass factor $\theta(p)$ measures bisectability rather than conventional concentration (Theorem~\ref{thm:softmax}). This result complements existing same-norm Lipschitz results, which concern different induced norms and are therefore not directly ordered with the present cross-norm identity.

\textbf{(2) Tokenwise local and bounded-domain attention bounds.} In a block-$\infty$/RMS geometry, row-stochastic attention mixing is nonexpansive. We use the exact softmax identity to derive a distribution-aware local Jacobian bound for MHA and, by taking $\theta\le 1$, a sequence-length-independent Lipschitz bound on bounded input sets (Theorem~\ref{thm:mha}, Corollary~\ref{cor:mha-bounded}). The decomposition separates a lower-degree value pathway from the softmax-mediated attention-weight pathway.

\textbf{(3) Residual-Jacobian structure under pre-LN and post-LN.} A pre-LN residual-sublayer Jacobian contains an additive identity term, whereas a post-LN residual-sublayer Jacobian does not (Theorems~\ref{thm:preln-gradient}--\ref{thm:postln-gradient}). We characterize the projection structure of the LayerNorm Jacobian and give a sufficient condition under which the LayerNorm-only term in the post-LN expansion contracts geometrically (the tokenwise condition of Lemma~\ref{lem:ln-projection} and its sequence-level lift in Corollary~\ref{cor:ln-sequence}). This condition is not empirically verified in our runs.

\textbf{(4) Sharpening versus bisectability.} In three Pre-LN early-training runs of $774$M-parameter GPT-2-style models on FineWeb, attention becomes substantially more concentrated (normalized entropy $0.987\!\to\!0.726$, mean row maximum weight $0.010\!\to\!0.174$) while the median lower-bound certificate for $\theta(p)$ remains near maximal at every sampled layer and checkpoint. A minority of rows develops substantially smaller certificates as their attention distributions become strongly concentrated, consistent with Corollary~\ref{cor:regimes} (sampled-row spread in Appendix~\ref{app:theta-traj}). Because the certificate lower-bounds the exact factor rowwise, a near-one per-layer median certifies near-maximal exact sensitivity for at least half of the sampled rows in that layer. Thus, in these runs, increased attention concentration coexists with near-maximal exact softmax sensitivity for at least half of the sampled rows at every sampled layer and checkpoint.

\section{Geometric Framework}
\label{sec:framework}

Analyzing transformer stability requires choosing a geometry matched to transformer operations. This section shows how token-aggregating norms can introduce explicit sequence-length dependence, then introduces the block-$\infty$/RMS norm used in the subsequent analysis.

\paragraph{Notation.}
We denote vectors $x \in \R^d$, matrices $X \in \R^{L \times d}$, and the probability simplex $\Delta^{L-1} = \{p \in \R^L : p_i \ge 0, \sum_i p_i = 1\}$. For $S \subseteq [L]$, we write $p(S) = \sum_{i \in S} p_i$.

The term ``stability'' is used in several non-equivalent senses in the transformer literature. To avoid ambiguity, Definition~\ref{def:stability} fixes the terminology used throughout the paper.

\begin{definition}[Stability concepts]
\label{def:stability}
For a transformer with hidden states $X_0,\ldots,X_N$ and layer maps $X_{\ell+1}=f_\ell(X_\ell)$, we distinguish:
\begin{enumerate}[leftmargin=*,itemsep=2pt]
\item \emph{Layerwise Lipschitz constant.} For a domain $\mathcal{D}$, $\operatorname{Lip}(f_\ell;\mathcal{D})=\sup_{X\ne X'\in\mathcal{D}}\|f_\ell(X)-f_\ell(X')\|/\|X-X'\|$; we write $L_{\mathrm{layer}}$ when the domain is clear. This measures worst-case sensitivity of a single layer on the stated domain.
\item \emph{End-to-end sensitivity.} $\|\partial X_N/\partial X_\ell\|$, the operator norm of the composed layer Jacobians from layer $\ell$ to the output. This measures how perturbations injected at layer $\ell$ amplify through the remaining depth.
\item \emph{Gradient-flow structure.} Backpropagation composes the transposed layer Jacobians $(\partial X_{\ell+1}/\partial X_\ell)^\top$, equivalently $(\partial X_N/\partial X_\ell)^\top$. We say a model has an \emph{additive identity gradient path} when the end-to-end Jacobian admits an additive identity term, i.e.\ $\prod_{k}(I+J_k)=I+\sum_k J_k+R_{\ge 2}$, where $R_{\ge 2}$ denotes the exact sum of all ordered products containing at least two sublayer Jacobians. The presence of the identity term is a structural statement about the expansion; it is not by itself a lower bound on the gradient norm, since the remaining terms may cancel part of the identity contribution.
\end{enumerate}
\end{definition}

The three notions are related but distinct: bounded $L_{\mathrm{layer}}$ does not imply bounded end-to-end sensitivity, and bounded end-to-end sensitivity in operator norm does not imply stable gradient flow on every direction (Lemma~\ref{lem:ln-projection} makes this precise for post-LN). Our results state which notion is being controlled in each case.

\subsection{Choosing a Geometry Aligned with Tokenwise Computation}
\label{subsec:why-norm}

The norm used in a stability analysis determines how perturbations are aggregated. A token-aggregating norm can introduce sequence-length dependence that is absent when sensitivity is measured at the level of the worst affected token. Two transformer operations motivate the geometry used here:

\emph{LayerNorm acts independently on each token.} The normalization formula $\mathrm{LN}(x_i)=\gamma\odot(x_i-\mu(x_i)\mathbf{1})/\sigma(x_i)+\beta$ depends only on the row $x_i$, never on $x_j$ for $j\ne i$. In particular, doubling a single token's magnitude leaves all other tokens' LayerNorm outputs unchanged.

\emph{Attention output is a row-stochastic mixture.} For each query position $i$, the output $(AV)_i=\sum_j A_{ij}v_j$ is a convex combination of value rows, so its magnitude is controlled by $\max_j\|v_j\|$, not by the aggregate $\sum_j\|v_j\|^2$.

A norm aligned with this structure should measure each token independently and aggregate by taking a maximum across tokens, not a sum or root-sum-of-squares. The block-$\infty$/RMS norm we introduce next does exactly this and is therefore aligned with the tokenwise structure above.

\paragraph{Length dependence is geometry-dependent.} The Frobenius norm $\|\Delta X\|_F=\sqrt{\sum_{i,j}\Delta X_{ij}^2}$ aggregates perturbations across all tokens and dimensions. For a row-stochastic matrix $A$,
\[
\|AX\|_F \le \|A\|_F\|X\|_F \le \sqrt{L}\,\|X\|_F.
\]
The block-$\infty$/RMS norm studies a different stability notion: the RMS magnitude of the worst affected token. For every $\Delta X\in\R^{L\times d}$,
\[
\frac{\|\Delta X\|_F}{\sqrt{Ld}} \;\le\; \|\Delta X\|_{\infty,\mathrm{rms}} \;\le\; \frac{\|\Delta X\|_F}{\sqrt{d}}.
\]
Thus transferring a bound between the two geometries can cost up to a factor $\sqrt{L}$. Our sequence-length-independent results should be interpreted as worst-token sensitivity statements, not as claims that every token-aggregated objective is independent of sequence length. Block-$\infty$/RMS bounds carry no explicit $L$ factor because attention mixing is nonexpansive in this geometry (Lemma~\ref{lem:stochastic}).

\subsection{Block-$\infty$/RMS Geometry}

\begin{definition}[Block-$\infty$/RMS norm]
\label{def:block-rms}
For $X = (x_1, \ldots, x_L)^\top \in \R^{L \times d}$ with rows $x_i\in\R^d$, $\|X\|_{\infty,\mathrm{rms}} = \max_i \|x_i\|_2/\sqrt{d}$.
\end{definition}

This is the maximum per-token RMS magnitude; the $\sqrt{d}$ normalization aligns with LayerNorm's standard deviation. For a single vector $x\in\R^d$ we write $\|x\|_{\mathrm{rms}}:=\|x\|_2/\sqrt{d}$, so that $\|X\|_{\infty,\mathrm{rms}}=\max_i\|x_i\|_{\mathrm{rms}}$. For a matrix $Z$ with arbitrary feature width (e.g.\ the $d_h$-dimensional per-head matrices $Q_h,K_h,V_h$), we also write the unnormalized row norm
\[
\|Z\|_{\infty,2} := \max_i \|z_i\|_2,
\]
so that for a model-width state $X\in\R^{L\times d}$, $\|X\|_{\infty,\mathrm{rms}}=\|X\|_{\infty,2}/\sqrt{d}$. Two facts about this geometry drive the layerwise bounds below (proofs in Appendix~\ref{app:geometry}).

\begin{lemma}[Attention mixing is nonexpansive]
\label{lem:stochastic}
For any row-stochastic $A\in\R^{L\times L}$ and $V\in\R^{L\times d}$, $\|AV\|_{\infty,\mathrm{rms}} \le \|V\|_{\infty,\mathrm{rms}}$.
\end{lemma}

\begin{lemma}[LayerNorm magnitude reset]
\label{lem:ln-output}
For any $x\in\R^d$, $\|\LN(x)\|_{\mathrm{rms}} \le \|\gamma\|_\infty + \|\beta\|_\infty$.
\end{lemma}

Lemma~\ref{lem:stochastic} shows that attention mixing does not increase the maximum per-token magnitude, while Lemma~\ref{lem:ln-output} shows that LayerNorm bounds the magnitude of each normalized token independently of the input scale. Appendix~\ref{app:geometry} also establishes $\mathrm{Lip}(\LN)\le \|\gamma\|_\infty/\sqrt{\epsilon}$.

\section{Exact Softmax Sensitivity}
\label{sec:softmax}

We characterize softmax sensitivity by the $\ell_\infty\to\ell_1$ operator norm of the Jacobian: $\ell_\infty$ on inputs measures the worst-case logit perturbation, while $\ell_1$ on outputs measures total redistributed probability mass.

\begin{definition}[Balanced-mass factor]
\label{def:theta}
For $p\in\Delta^{L-1}$, $\theta(p) = 4\,\max_{S\subseteq[L]} p(S)(1-p(S))\in[0,1]$, where $p(S)=\sum_{i\in S}p_i$.
\end{definition}

$\theta(p)$ measures how evenly probability mass can be \emph{bisected}: the product $p(S)(1-p(S))$ is maximized at $p(S)=1/2$, giving $\theta=1$, and becomes small only when every subset mass is far from $1/2$, with $\theta=0$ only in the exactly one-hot case. The quantity $\theta(p)$ is distinct from concentration measures: $p=(0.4,0.1,0.4,0.1)$ has lower entropy than uniform but $\theta=1$ because $p(\{2,3\})=0.5$. Thus a non-uniform attention distribution can still have maximal $\theta$.

To avoid ambiguity about where the derivative is taken, define the tempered softmax $s_\tau(u)=\softmax(u/\tau)$ and let $J_\tau(u)=D_u\,s_\tau(u)$ denote its Jacobian with respect to the \emph{unscaled} logits $u$.

\begin{theorem}[Exact softmax Jacobian norm]
\label{thm:softmax}
Let $u\in\R^L$, $\tau>0$, and $p = s_\tau(u)$. Then $J_\tau(u) = \frac{1}{\tau}(\Diag(p) - pp^\top)$ and
\[
\|J_\tau(u)\|_{\infty \to 1} \;=\; \frac{\theta(p)}{\tau}.
\]
The maximum is achieved by $x^*\in\{\pm 1\}^L$ with $x^*_i=+1$ for $i\in S^*$ and $-1$ otherwise, where $S^*$ attains the maximum in Definition~\ref{def:theta}.
\end{theorem}

The proof in Appendix~\ref{app:softmax} reduces the operator norm to a maximum over sign vectors and identifies the maximizing subset. The resulting value depends exactly on the probability mass profile through $\theta(p)$. Section~\ref{sec:layers} then propagates this row-level identity through the separate MHA projections.

\begin{remark}[Combinatorial interpretation and computation]
\label{rem:maxcut}
The matrix $\Diag(p)-pp^\top$ is the graph Laplacian of the complete graph on $[L]$ with edge weights $w_{ij}=p_ip_j$ ($i\neq j$) for $i\neq j$. For this matrix the identification is self-contained: at a sign vector $x\in\{\pm1\}^L$ with positive set $S$, a direct computation gives $x^\top(\Diag(p)-pp^\top)x = 1-(2p(S)-1)^2 = 4\,p(S)(1-p(S)) = 4\operatorname{Cut}_w(S,\bar S)$, where $\operatorname{Cut}_w(S,\bar S)=\sum_{i\in S,\,j\notin S}p_ip_j = p(S)(1-p(S))$ is the weight of the cut $(S,\bar S)$. Combined with Theorem~\ref{thm:softmax}, whose maximizer is a sign vertex attaining $4\,p(S^*)(1-p(S^*))$, this shows that $\theta(p)$ is four times the maximum weighted cut value of this rank-one-weighted graph, equivalently the maximum of the quadratic cut objective over sign vectors. For attention, this identifies softmax sensitivity with the bisectability of the attention distribution.

The same interpretation also explains the computational difficulty of evaluating $\theta(p)$ exactly: maximizing the cut is equivalent to choosing a subset of probabilities whose mass is closest to $1/2$, a \textsc{Subset-Sum}/\textsc{Partition}-type optimization problem. Accordingly, our large-$L$ experiments report feasible-subset lower-bound certificates rather than exact values (Section~\ref{sec:experiments}).
\end{remark}

\begin{corollary}[Sensitivity regimes]
\label{cor:regimes}
(Proof in Appendix~\ref{app:softmax}.)
\begin{enumerate}[leftmargin=*,itemsep=2pt]
\item Uniform: $p = \frac{1}{L}\mathbf{1}$ gives $\theta(p) = 1$ for even $L$ and $\theta(p) = 1 - 1/L^2$ for odd $L$.
\item One-hot: $p = e_i$ gives $\theta(p) = 0$ (minimum sensitivity).
\item Dominant atom: If the top token has mass $1-\kappa$ with $\kappa \le 1/2$, then $\theta(p) = 4\kappa(1-\kappa)$.
\item Top-$k$ uniform: Mass $1/k$ on $k$ tokens gives $\theta(p) = 1 - (1 - 2\lfloor k/2 \rfloor/k)^2$.
\item Maximum-weight lower bound: For $p_{\max}:=\max_i p_i$, $\theta(p) \ge 1 - p_{\max}^2$.
\end{enumerate}
\end{corollary}

\paragraph{Implications for training.}
At random initialization, attention distributions are often near-uniform \citep{noci2022signal,naitsaada2025attention}. For rows with sufficiently large causal support, this places $\theta(p)$ near one and the softmax Jacobian norm near its maximum $1/\tau$. The earliest causal rows are an exception: the position-$0$ row has support one and therefore $\theta=0$.

As attention becomes more concentrated during training, however, $\theta(p)$ can remain large. A non-uniform distribution may still admit a subset with mass close to $1/2$, and the maximum-weight lower bound $\theta(p)\ge1-p_{\max}^2$ shows that moderate concentration alone cannot make $\theta$ small. Section~\ref{sec:experiments} measures this distinction directly.

We next propagate the exact row-level softmax sensitivity through multi-head attention and complete transformer layers.
\section{Layerwise Stability Analysis}
\label{sec:layers}

We next propagate the exact row-level softmax sensitivity into MHA and complete-layer bounds, then examine how residual-sublayer Jacobians compose through depth. Let $X_\ell\in\R^{L\times d}$ denote the hidden state at layer $\ell$ for $\ell = 0, 1, \ldots, N$, where $N$ is the total number of layers. The analysis makes explicit how $\theta(p)/\tau$, input magnitude, projection norms, and normalization placement enter these bounds. Proofs for this section appear in Appendix~\ref{app:layers}.

\subsection{Multi-Head Attention Local and Bounded-Domain Sensitivity}

Multi-head attention maps input $U \in \R^{L \times d}$ to output $\MHA(U) = [O_1 \| \cdots \| O_H] W^O$, where head $h$ computes:
\[
Q_h = UW_h^Q, \quad K_h = UW_h^K, \quad V_h = UW_h^V,
\]
\[
A_h = \softmax(Q_h K_h^\top / (\tau\sqrt{d_h})), \quad O_h = A_h V_h.
\]
In standard training $\tau = 1$, but we retain the explicit $\tau$ to make the temperature dependence of the softmax-sensitivity factor visible.

\paragraph{Architectural assumptions.}
The analysis assumes bias-free projections $W_h^Q,W_h^K,W_h^V,W^O$ and a dropout-disabled attention map: with attention dropout active, the applied mixing matrix is no longer row-stochastic and Lemma~\ref{lem:stochastic} would not apply as stated. Our experimental models satisfy the bias-free assumption. Attention-sensitivity measurements are taken in evaluation mode, whereas training losses and gradient norms are recorded from the training pass.

For a fixed causal or other fixed mask, Theorem~\ref{thm:softmax} applies rowwise on the unmasked support. At position $t$, the causal row is a softmax over its $t+1$ visible logits; masked coordinates have zero derivative, and $\theta$ is computed on the support distribution.

To bound sensitivity, we decompose the differential into two pathways. The \emph{value pathway} captures how input perturbations affect values $V$, which are then averaged by fixed attention weights. The \emph{attention-weight pathway} captures how input perturbations change attention weights $A$, which then act on fixed values. For $U\in\R^{L\times d}$, define the local induced Jacobian norm
\[
\Lambda_{\MHA}(U) := \|D\MHA(U)\|_{\infty,\mathrm{rms}\to\infty,\mathrm{rms}}.
\]

\begin{theorem}[MHA local Jacobian bound]
\label{thm:mha}
Let $\bar{B}_U = \|U\|_{\infty,\mathrm{rms}}\sqrt{d}$ and $\tilde{\theta}_h(U) = \max_i \theta(A_h(U)[i,:])$. Then
\[
\Lambda_{\MHA}(U) \le \|W^O\|_2 \sum_{h=1}^H \left(\|W_h^V\|_2 + \frac{\tilde{\theta}_h(U)}{\tau} \Phi_h\right),
\]
where $\Phi_h = \frac{2\bar{B}_U^2}{\sqrt{d_h}}\|W_h^Q\|_2\|W_h^K\|_2\|W_h^V\|_2$. The first term $\|W_h^V\|_2$ is the value pathway contribution; the second term $\frac{\tilde{\theta}_h(U)}{\tau}\Phi_h$ is the attention-weight pathway contribution.
\end{theorem}

\begin{corollary}[MHA Lipschitz bound on a bounded input set]
\label{cor:mha-bounded}
For $B\ge 0$, let $\mathcal{D}_B = \{U : \|U\|_{\infty,\mathrm{rms}} \le B\}$. Then $\operatorname{Lip}(\MHA;\mathcal{D}_B) \le L_{\MHA}^{(B)}$, where
\[
L_{\MHA}^{(B)} := \|W^O\|_2 \sum_{h=1}^H \left(\|W_h^V\|_2 + \frac{2dB^2}{\tau\sqrt{d_h}}\|W_h^Q\|_2\|W_h^K\|_2\|W_h^V\|_2\right).
\]
\end{corollary}

The corollary follows from Theorem~\ref{thm:mha} by using $\tilde\theta_h\le 1$ and integrating the Jacobian along the line segment between two points of the convex set $\mathcal{D}_B$ (Appendix~\ref{app:layers}). The two bounds serve different purposes. The local result retains the distribution-dependent factor $\tilde\theta_h(U)$ because the attention distribution at $U$ is known. The bounded-domain result instead uses the uniform inequality $\theta\le1$, since attention distributions along the interpolation path are not known in advance.

The local and bounded-domain bounds contain no explicit sequence-length factor, a consequence of the tokenwise-aligned geometry in Section~\ref{subsec:why-norm}. They are not dimension-free: for block-RMS-bounded inputs with $B=O(1)$, the attention-weight coefficient contains $d/\sqrt{d_h}$ before accounting for the projection norms. The attention-weight coefficient is quadratic in the input magnitude $\bar{B}_U$ and contains the quartic projection product $\|W^O\|_2\|W^Q\|_2\|W^K\|_2\|W^V\|_2$. The complete MHA bound also contains the lower-degree value pathway through $W^V$ and $W^O$. Therefore the quartic attention-weight term alone does not imply a universal depth-scaling rule for the full MHA block.

\subsection{Per-Layer Lipschitz Bounds}

Both pre-LN ($X_{\ell+1}=X_\ell+\MHA(\LN(X_\ell))$) and post-LN ($X_{\ell+1}=\LN(X_\ell+\MHA(X_\ell))$) keep the MHA input magnitude bounded across layers. Pre-LN does so structurally, since the LayerNorm before MHA caps $\|\LN(X_\ell)\|_{\infty,\mathrm{rms}}\le \|\gamma\|_\infty+\|\beta\|_\infty$ by Lemma~\ref{lem:ln-output} regardless of how large the residual stream grows; post-LN does so because $X_\ell$ for $\ell\ge 1$ is itself a LayerNorm output. Because LayerNorm parameters may vary across layers, define the uniform constants
\[
B_{\LN} := \max_j\bigl(\|\gamma_j\|_\infty+\|\beta_j\|_\infty\bigr),
\qquad
L_{\LN} := \max_j\frac{\|\gamma_j\|_\infty}{\sqrt{\epsilon_j}},
\]
where the maxima run over the LayerNorm modules under consideration (Lemma~\ref{lem:ln-lip} bounds each module's Lipschitz constant by its own $\|\gamma_j\|_\infty/\sqrt{\epsilon_j}$). Let $L_{\MHA,\ell}^{(B_{\LN})}$ denote the bounded-domain constant of Corollary~\ref{cor:mha-bounded} evaluated with the layer-$\ell$ projection matrices and $B=B_{\LN}$, and let $L_{\FFN,\ell}$ denote the layer-$\ell$ FFN Lipschitz constant.

\begin{theorem}[Pre-LN per-layer bound]
\label{thm:preln-full}
For a complete pre-LN layer,
\[
L_{\mathrm{layer},\ell}^{\mathrm{pre}}\;\le\;\bigl(1+L_{\LN}\,L_{\MHA,\ell}^{(B_{\LN})}\bigr)\bigl(1+L_{\LN}\,L_{\FFN,\ell}\bigr).
\]
The bound has no explicit sequence-length dependence. If the relevant weight norms and LayerNorm parameters are uniformly bounded across layers, the right-hand side is uniform in $\ell$.
\end{theorem}

\begin{theorem}[Post-LN bound on the normalized-state domain]
\label{thm:postln-full}
For $\ell\ge 1$, the incoming state is a LayerNorm output and satisfies $\|X_\ell\|_{\infty,\mathrm{rms}}\le B_{\LN}$. On this domain,
\[
L_{\mathrm{layer},\ell}^{\mathrm{post}}\;\le\;L_{\LN}^2\bigl(1+L_{\MHA,\ell}^{(B_{\LN})}\bigr)\bigl(1+L_{\FFN,\ell}\bigr).
\]
The first layer requires a separate bound on its input magnitude.
\end{theorem}

Both architectures therefore admit finite, sequence-length-independent upper bounds on the stated domains. These bounds do not establish that their actual per-layer sensitivities are equal, nor do they by themselves explain the empirical optimization gap. We next examine the algebraic structure of residual-sublayer Jacobians through depth.

\subsection{Gradient Flow Analysis}

We now examine how gradients propagate through depth, where pre-LN and post-LN differ structurally. Let $J_{\MHA}(X)$ and $J_{\LN}(Y)$ denote the Jacobians of MHA and LayerNorm evaluated at inputs $X$ and $Y$, respectively. The layer Jacobian $\partial X_{\ell+1}/\partial X_\ell$ determines how gradients flow backward through layer $\ell$. Throughout this subsection and its proofs, ordered matrix products are written with later layers on the left:
\[
\prod_{k=\ell}^{m-1} A_k \;:=\; A_{m-1}A_{m-2}\cdots A_\ell,
\]
matching the composition order of the chain rule.

When the argument of $J_{\LN}$ is a matrix $Y\in\R^{L\times d}$, $J_{\LN}(Y)$ denotes the Jacobian of the rowwise map $Y\mapsto\LN(Y)$, which is block diagonal across tokens with per-token blocks $J_{\LN}(y_i)$; we write $\mathcal{J}_{\LN,k}(Y)$ for this same sequence-level Jacobian where the distinction from its per-token blocks matters (Corollary~\ref{cor:ln-sequence} in Appendix~\ref{app:geometry} makes the tokenwise-to-sequence lift precise).

For clarity, we analyze one attention residual sublayer; the same structural distinction between the presence and absence of a bare identity term applies separately to the FFN residual sublayer. The algebra below does not by itself provide a lower or upper bound on the full end-to-end gradient norm.

\paragraph{Pre-LN gradient structure.} In pre-LN, $X_{\ell+1} = X_\ell + \MHA(\LN(X_\ell))$, so the layer Jacobian is:
\[
\frac{\partial X_{\ell+1}}{\partial X_\ell} = I + J_{\MHA}(\LN(X_\ell)) \cdot J_{\LN}(X_\ell).
\]
The identity $I$ appears as an \emph{additive} term.

\begin{theorem}[Pre-LN: identity gradient path]
\label{thm:preln-gradient}
Define $J_k := J_{\MHA}(\LN(X_k)) \cdot J_{\LN}(X_k)$. Then:
\begin{enumerate}
\item[(i)] Each sublayer Jacobian is bounded: $\|J_k\| \le L_{\MHA,k}^{(B_{\LN})} \cdot L_{\LN} =: C_k$, with no explicit dependence on $N$ or $L$.
\item[(ii)] The end-to-end Jacobian expands as:
\[
\frac{\partial X_m}{\partial X_\ell} = \prod_{k=\ell}^{m-1}(I + J_k) = I + \sum_{k=\ell}^{m-1} J_k + R_{\ge 2},
\]
where $R_{\ge 2}$ is the exact sum of all ordered products containing at least two sublayer Jacobians.
\item[(iii)] The expansion contains an additive identity term that does not itself accumulate a multiplicative depth factor.
\end{enumerate}
\end{theorem}

The remainder $R_{\ge 2}$ is an exact algebraic object, not a uniformly small term in depth. Part (iii) is a structural statement only; it does not imply $\|\partial X_m/\partial X_\ell\|\ge1$, because the remaining terms can cancel the identity contribution. The alternating-sign scalar chain $x_{k+1}=x_k+(-1)^k a\,x_k$ provides an explicit example: its end-to-end Jacobian decays as $(1-a^2)^{(m-\ell)/2}$. The point is therefore not a gradient lower bound, but the existence of one term in the expansion that contains no sublayer Jacobian factors.

\paragraph{Post-LN gradient structure.} In post-LN, $X_{\ell+1} = \LN(X_\ell + \MHA(X_\ell))$, so the LayerNorm Jacobian left-multiplies the residual-block Jacobian:
\[
\frac{\partial X_{\ell+1}}{\partial X_\ell} = J_{\LN}(Y_\ell) \cdot (I + J_{\MHA}(X_\ell)),
\]
where $Y_\ell = X_\ell + \MHA(X_\ell)$.

\begin{theorem}[Post-LN: no identity gradient path]
\label{thm:postln-gradient}
The end-to-end Jacobian is:
\begin{enumerate}
\item[(i)] $\displaystyle\frac{\partial X_m}{\partial X_\ell} = \prod_{k=\ell}^{m-1} J_{\LN}(Y_k) (I + J_{\MHA}(X_k))$.
\item[(ii)] The LayerNorm Jacobian $J_{\LN}(y)$ is projection-like: it annihilates the constant direction ($J_{\LN}(y)\mathbf{1} = 0$), is bounded on the centered subspace, and has additional shrinkage along the normalized centered input direction.
\item[(iii)] Expanding the product yields $\prod_{k=\ell}^{m-1} \mathcal{J}_{\LN,k}(Y_k) + \bigl(\text{terms involving } J_{\MHA}\bigr)$. We call the first term the \emph{LayerNorm-only term}; the expansion contains no bare additive identity term.
\end{enumerate}
\end{theorem}

Lemma~\ref{lem:ln-projection} (Appendix~\ref{app:geometry}) shows that each per-token LayerNorm Jacobian $J_{\LN,k}(y)$ annihilates the constant direction $\mathbf{1}$ and, on the centered subspace orthogonal to $\mathbf{1}$, has operator norm bounded by $\|\gamma_k\|_\infty/\sigma_k(y)$, with an additional shrinkage factor $\rho(y)<1$ along the normalized input direction when $y$ is non-constant. Because LayerNorm acts tokenwise, the sequence-level Jacobian is block diagonal across tokens, and Corollary~\ref{cor:ln-sequence} lifts the per-token bound to the full $L\times d$ state: writing $a_{k,i}=\|\gamma_k\|_\infty/\sigma_k(y_{k,i})$ for token $i$ at layer $k$, the LayerNorm-only term satisfies $\|\prod_{k=\ell}^{m-1}\mathcal{J}_{\LN,k}(Y_k)\|_2\le\max_i\prod_k a_{k,i}$, which decays geometrically whenever the worst-token factor $\max_i a_{k,i}$ is at most $1-\eta$ at every layer.

Corollary~\ref{cor:ln-sequence} therefore gives a sufficient geometric-contraction condition for the LayerNorm-only term. It does not bound the remaining terms that contain MHA Jacobians, so Theorem~\ref{thm:postln-gradient} should be read as a structural decomposition rather than as a contraction theorem for the full post-LN Jacobian. We do not test the sufficient condition empirically.

\subsection{Discussion}

The per-layer bounds do not themselves separate pre-LN from post-LN; the distinction appears in Jacobian composition. The pre-LN expansion contains a bare additive identity term, whereas the post-LN expansion does not. Corollary~\ref{cor:ln-sequence} gives a sufficient contraction condition only for the LayerNorm-only post-LN term. These structural observations motivate the gradient-profile comparison in Section~\ref{sec:experiments}, but they do not determine the full end-to-end gradient in either architecture.

\section{Empirical Analysis}
\label{sec:experiments}

We evaluate the theory in three-seed early-training runs of $774$M-parameter GPT-2-style transformers on FineWeb \citep{penedo2024fineweb}. Each run covers $2000$ optimizer updates ($131$M tokens), with warmup completing at one quarter of the run; we therefore make no claim about converged behavior. We ask three questions: (i) does the exact softmax identity agree numerically with exhaustive computation at small $L$? (ii) how do pre-LN and post-LN optimization and gradient profiles differ in this window? and (iii) how does the large-$L$ certificate $\thetacert$ evolve as attention becomes more concentrated? Full experimental details are in Appendix~\ref{app:experiments}.

\paragraph{Certificates for $\theta(p)$ at large $L$.}
For large $L$, we do not compute $\theta(p)$ exactly (by Remark~\ref{rem:maxcut}, exact computation is a \textsc{Partition}-type subset optimization). Given the greedy-selected subset $S_{\mathrm{g}}$ described in Appendix~\ref{app:experiments}, define
\[
\thetacert(p) := 4\,p(S_{\mathrm{g}})(1-p(S_{\mathrm{g}})).
\]
Because $S_{\mathrm{g}}$ is feasible, $\thetacert(p) \le \theta(p) \le 1$. Thus a certificate close to one proves that the exact $\theta(p)$ is close to one; all large-$L$ empirical quantities below are certificates $\thetacert$, not exact values.

\paragraph{Attention measurements.}
At each checkpoint, we run a separate evaluation-mode forward pass (dropout disabled) and capture attention probabilities at layers $\{0,N/2,N-1\}$. From each layer, we sample $512$ rows pooled over batch and heads, restricting query positions to $t\ge16$ so that the causal support $t+1$ is bounded away from the degenerate early rows. We compute $\thetacert$ for each sampled row on its support distribution and record the 5th, 25th, 50th, 75th, and 95th percentiles per layer, unrounded. For the same sampled rows we also record the row maximum $p_{\max}=\max_j p_j$ and the normalized entropy
\[
H_{\mathrm{norm}}(p) \;=\; \frac{-\sum_j p_j\log p_j}{\log|\operatorname{supp}(p)|},
\]
where $|\operatorname{supp}(p)|=t+1$ is the causal support at query position $t$. Reported $\thetacert$ aggregates take, at each checkpoint, the mean across sampled layers of the per-layer median (averaged across seeds), together with a \emph{percentile envelope} running from the smallest per-layer 5th percentile to the largest per-layer 95th percentile over layers and seeds. The envelope is not a pooled 5th--95th percentile interval of the row population. Reported entropy and $p_{\max}$ trajectories are the mean across sampled layers of the per-layer sampled-row mean, averaged across seeds.

\paragraph{Relation to Theorem~\ref{thm:mha}.}
Because rows are pooled over heads, the maximum sampled certificate lower-bounds $\max_h\tilde\theta_h(U)$, where $\tilde\theta_h(U)=\max_i\theta(A_h[i,:])$, rather than each head-specific $\tilde\theta_h$. A pooled maximum near one therefore certifies that at least one head has a near-maximal distribution-dependent factor. The sampled-row quantiles describe how common near-bisectable rows are across the pooled population. A per-head certificate, the maximum over sampled rows from a single head, would lower-bound that head's $\tilde\theta_h$ individually; our logged statistics pool heads.

\paragraph{Other metrics.}
Loss is mean cross-entropy per token, averaged over exactly the micro-batches of each optimizer update (sequences are packed, so no padding tokens exist); gradient norm is the global parameter-gradient $\ell_2$ norm after accumulation and before clipping; layerwise gradient norms are per-block parameter-gradient norms, a descriptive quantity distinct from the activation Jacobians analyzed in Section~\ref{sec:layers}. A gradient spike is an optimizer update whose pre-clipping norm exceeds twice the median of the available preceding updates, using a window of at most $50$ updates; the first $10$ updates are excluded. Spikes are computed offline from the logged per-update series. All reported $\pm$ values are sample standard deviations across the three seeds (denominator $n-1$).

\paragraph{Numerical checks and large-$L$ certificates.}
For $L\in\{8,16\}$, we independently compute both sides of $\|J_\tau(u)\|_{\infty \to 1} = \theta(p)/\tau$ by exhaustive enumeration and obtain relative error below $10^{-13}$ over $500$ random logit vectors per tested length. For larger $L$, we do not claim independent numerical verification of the operator norm; we compute the feasible-subset certificate $\thetacert(p)\le\theta(p)$, and certificates close to one establish that the exact sensitivity factor is also close to its maximum. The full procedure is in Appendix~\ref{app:experiments}.

\paragraph{Optimization outcomes.}
Figure~\ref{fig:combined} reports the logged trajectory. Over the final $50$ updates, Pre-LN reaches $4.509\pm0.005$ across the three seeds. Post-LN outcomes split sharply by seed: one seed reaches $4.98$, while two plateau near $7.67$ from roughly update $250$ onward (mean $6.771\pm1.551$). For each matched seed, Pre-LN has lower final-window loss than Post-LN. These are early-training results, not claims about converged performance.

\paragraph{Gradient spikes.}
Post-LN registers $35.3\pm11.7$ spike updates per run under the protocol rule, compared with $9.0\pm1.0$ for Pre-LN. The largest observed pre-clipping norm across the three runs is $112$ for Post-LN and $35$ for Pre-LN. Most of the Post-LN excess occurs during the first $\sim150$ updates. Several spike updates coincide across seeds and architectures (e.g.\ updates $1057$--$1059$ and $1459$--$1460$); because the runs share the same data order, these synchronized spikes are consistent with shared data-order effects and should not be attributed solely to architecture.

\paragraph{Layerwise gradient profiles.}
The layerwise profiles also differ sharply. In the two plateaued Post-LN runs, the last-checkpoint per-block gradient norms are at most $7\times10^{-6}$ across layers $0$--$33$, four orders of magnitude below the final blocks ($0.04$ at layer $35$). For the seed displayed in Figure~\ref{fig:combined}(d) (seed~$123$), the non-plateaued Post-LN run still shows a strong increase toward the output (mean $0.023$ over layers $0$--$29$ versus $0.14$--$0.20$ at the final blocks). For the corresponding Pre-LN run, the profile varies more smoothly across depth ($0.02$--$0.06$ after an initial $0.15$ at layer $0$). These parameter-gradient profiles are descriptive and do not test the sufficient contraction condition of Corollary~\ref{cor:ln-sequence}.

\begin{figure}[t]
\centering
\includegraphics[width=0.9\columnwidth]{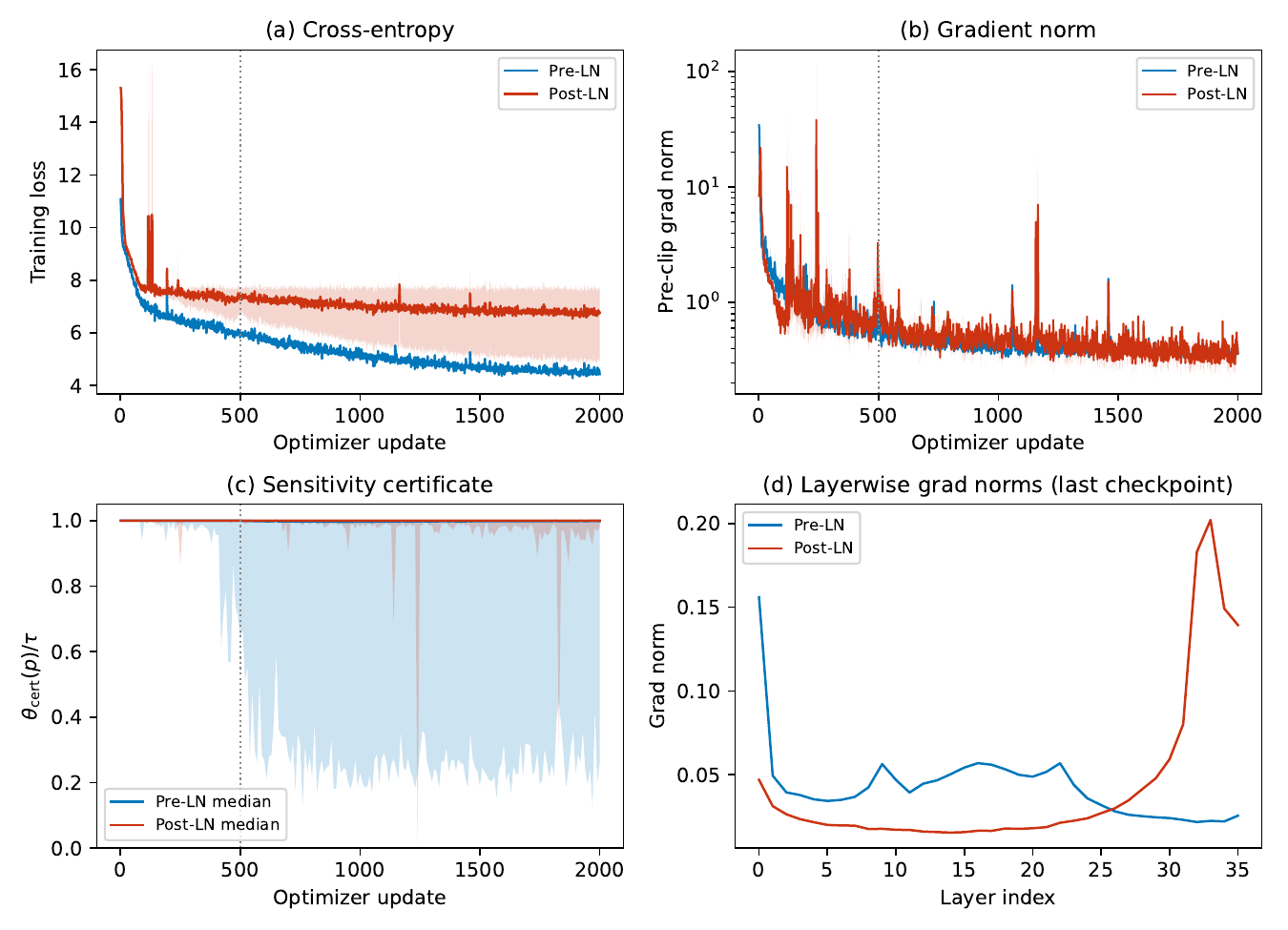}
\caption{Early-training dynamics for $774$M-parameter transformers over the logged trajectory. \textbf{(a)} Cross-entropy and \textbf{(b)} pre-clipping gradient norm: solid lines are the per-update mean across the three seeds, shaded regions span the per-update minimum to maximum across seeds. Panel (b) uses a logarithmic $y$-axis. \textbf{(c)} Lower-bound certificate $\thetacert(p)/\tau$: solid lines show the mean across sampled layers and seeds of the per-layer medians; shaded regions show the envelope from the minimum per-layer 5th percentile to the maximum per-layer 95th percentile across layers and seeds. \textbf{(d)} Per-block parameter-gradient norms at the last logged checkpoint, shown for a single seed (seed~$123$, the non-plateaued Post-LN seed); in the two plateaued Post-LN runs these norms are at most $7\times10^{-6}$ through layers $0$--$33$. The dotted vertical line in (a)--(c) marks the end of warmup at update $500$. The gradient profiles are descriptive; they do not test the sufficient LayerNorm-contraction condition of Corollary~\ref{cor:ln-sequence}.}
\label{fig:combined}
\end{figure}
\FloatBarrier

\paragraph{Attention concentration increases while the median sensitivity certificate remains near one.}
Across the three Pre-LN seeds, normalized attention entropy falls from $0.987$ to $0.726$, while mean row maximum weight rises from $0.010$ to $0.174$ (means over the sampled layers; Appendix~\ref{app:theta-traj}). Despite this substantial increase in concentration, the median certificate $\thetacert(p)/\tau$ remains at or above $0.995$ at every sampled layer and checkpoint (last-checkpoint mean of the per-layer medians, $0.9987$).

The median hides a substantial lower tail. At the middle sampled layer, the 5th percentile of $\thetacert$ falls to $0.26$--$0.58$ across seeds, and the corresponding rows have maximum weights approaching $0.9$. Corollary~\ref{cor:regimes} explains why the median and tail can behave differently: moderate $p_{\max}$ leaves $\theta$ bounded away from zero by $\theta(p)\ge1-p_{\max}^2$, while rows with $p_{\max}\ge1/2$ satisfy the exact dominant-atom formula $\theta(p)=4p_{\max}(1-p_{\max})$. A low greedy certificate alone does not upper-bound the exact $\theta$, but the dominant-atom formula does characterize the strongly concentrated rows directly.

Attention in the two plateaued Post-LN runs remains near uniform: their sampled-row maximum weights closely match the uniform-support baseline in these two summary statistics (median $0.0020$ vs.\ $0.0019$; mean $0.0041$ vs.\ $0.0041$), and their certificates remain near one. The Post-LN seed that trains does develop concentrated attention at the deeper sampled layers. In our runs, substantial attention concentration therefore coexists with near-maximal exact softmax sensitivity for at least half of the sampled rows. Rows approaching the dominant-atom regime exhibit lower exact sensitivity. These observations do not identify the cause of the pre-LN/post-LN optimization gap.

\FloatBarrier
\section{Conclusion}
\label{sec:conclusion}

We derived the exact identity $\|J_\tau(u)\|_{\infty\to1}=\theta(p)/\tau$ for the tempered-softmax Jacobian (Theorem~\ref{thm:softmax}), where $\theta(p)$ measures bisectability rather than conventional concentration. In a block-$\infty$/RMS geometry aligned with worst-token perturbations, the identity yields a distribution-aware local Jacobian bound for multi-head attention and a sequence-length-independent Lipschitz bound on bounded input sets. The MHA decomposition separates a lower-degree value pathway from a softmax-mediated attention-weight pathway and makes the dependence on input magnitude, width, temperature, and projection norms explicit.

We also identified a structural distinction between normalization placements. A pre-LN residual-sublayer Jacobian contains an additive identity term, whereas a post-LN residual-sublayer Jacobian does not. Under an explicit sufficient condition, the LayerNorm-only term in the post-LN expansion contracts geometrically. This result does not establish contraction of the full post-LN Jacobian, and the sufficient condition is not tested in our experiments.

Across the three Pre-LN early-training runs, attention becomes substantially more concentrated while each sampled layer retains a near-one median certificate. This certifies near-maximal exact sensitivity for at least half of the sampled rows within each layer. A minority of strongly concentrated rows enters the lower-sensitivity dominant-atom regime characterized by Corollary~\ref{cor:regimes}. These results show that concentration and bisectability capture different properties of an attention distribution.

\textbf{Limitations.} The block-$\infty$/RMS geometry measures worst-token rather than token-aggregated sensitivity; the bounds retain explicit width and parameter-norm dependence; the post-LN contraction condition is unverified empirically; large-$L$ experiments use lower-bound certificates rather than exact subset optimization; and the empirical comparisons cover only a $2000$-update early-training window at a single model configuration ($131$M tokens per run). Extensions to RMSNorm, grouped-query attention, MoE, and other modern architectures remain open.
\section*{Acknowledgments}
The author thanks Halil Talha Ertan for identifying a counterexample to the necessity direction of an earlier depth-scaling theorem and a parity oversight in the uniform case of Corollary~\ref{cor:regimes}. The author also thanks the NeurIPS 2026 reviewers for identifying an error in the earlier depth-scaling derivation and for suggestions that led to the maximum-weight lower bound and the attention-sink discussion. The depth-scaling claims have been removed from this version. Any remaining errors are the author's own.

\bibliographystyle{plainnat}
\bibliography{references_revised}

\newpage
\appendix
\section{Proofs for Section~\ref{sec:framework}: Geometric Framework}
\label{app:geometry}

This appendix provides proofs for the geometric framework developed in Section~\ref{sec:framework}.

\begin{proof}[Proof of Lemma~\ref{lem:stochastic} (Attention mixing is nonexpansive)]
Let $A \in \R^{L \times L}$ be row-stochastic, so $A_{ij} \ge 0$ and $\sum_{j=1}^L A_{ij} = 1$ for all $i$. Let $V \in \R^{L \times d}$ with rows $v_1, \ldots, v_L \in \R^d$.

The $i$-th row of $AV$ is
\[
(AV)_i = \sum_{j=1}^L A_{ij} v_j.
\]
Since $A_{ij} \ge 0$ and $\sum_j A_{ij} = 1$, this is a convex combination of the rows of $V$. By the triangle inequality:
\begin{align*}
\|(AV)_i\|_2 &= \left\|\sum_{j=1}^L A_{ij} v_j\right\|_2 \le \sum_{j=1}^L A_{ij} \|v_j\|_2 \\
&\le \sum_{j=1}^L A_{ij} \max_k \|v_k\|_2 = \max_k \|v_k\|_2.
\end{align*}
Taking the maximum over $i$ and dividing by $\sqrt{d}$:
\begin{align*}
\|AV\|_{\infty,\mathrm{rms}} &= \max_i \frac{\|(AV)_i\|_2}{\sqrt{d}} \\
&\le \frac{\max_k \|v_k\|_2}{\sqrt{d}} = \|V\|_{\infty,\mathrm{rms}}.
\end{align*}
\end{proof}

\begin{proof}[Proof of Lemma~\ref{lem:ln-output} (LayerNorm magnitude reset)]
LayerNorm computes
\[
\LN(x) = \gamma \odot \frac{x - \mu(x) \mathbf{1}}{\sigma(x)} + \beta,
\]
where $\mu(x) = \frac{1}{d}\sum_{j=1}^d x_j$, $\sigma(x) = \sqrt{\frac{1}{d}\sum_{j=1}^d (x_j - \mu(x))^2 + \epsilon}$, and $\mathbf{1}$ is the all-ones vector.

Let $\hat{x} = (x - \mu(x)\mathbf{1})/\sigma(x)$ denote the normalized vector. We have
\[
\|\hat{x}\|_2^2 = \sum_{j=1}^d \hat{x}_j^2 = \sum_{j=1}^d \frac{(x_j - \mu(x))^2}{\sigma(x)^2} = \frac{d \cdot (\sigma(x)^2 - \epsilon)}{\sigma(x)^2} \le d.
\]
Thus $\|\hat{x}\|_2 \le \sqrt{d}$, which gives $\|\hat{x}\|_{\mathrm{rms}} = \|\hat{x}\|_2/\sqrt{d} \le 1$.

For the full LayerNorm output:
\[
\|\LN(x)\|_{\mathrm{rms}} = \frac{\|\gamma \odot \hat{x} + \beta\|_2}{\sqrt{d}} \le \frac{\|\gamma \odot \hat{x}\|_2 + \|\beta\|_2}{\sqrt{d}}.
\]
Since $|(\gamma \odot \hat{x})_j| = |\gamma_j||\hat{x}_j| \le \|\gamma\|_\infty |\hat{x}_j|$, we have $\|\gamma \odot \hat{x}\|_2 \le \|\gamma\|_\infty \|\hat{x}\|_2 \le \|\gamma\|_\infty \sqrt{d}$. Similarly, $\|\beta\|_2 \le \|\beta\|_\infty \sqrt{d}$. Therefore:
\[
\|\LN(x)\|_{\mathrm{rms}} \le \frac{\|\gamma\|_\infty \sqrt{d} + \|\beta\|_\infty \sqrt{d}}{\sqrt{d}} = \|\gamma\|_\infty + \|\beta\|_\infty.
\]
\end{proof}

\begin{lemma}[LayerNorm Lipschitz constant]
\label{lem:ln-lip}
The Lipschitz constant of LayerNorm satisfies $\mathrm{Lip}(\LN) \le \|\gamma\|_\infty / \sqrt{\epsilon}$, where $\epsilon > 0$ is the numerical stability floor in the denominator.
\end{lemma}

\begin{proof}
The Jacobian of LayerNorm at $x$ is
\[
J_{\LN}(x) = \frac{\Diag(\gamma)}{\sigma(x)} \left(I - \frac{\mathbf{1}\mathbf{1}^\top}{d} - \frac{(x - \mu\mathbf{1})(x - \mu\mathbf{1})^\top}{d \sigma(x)^2}\right),
\]
where $\mu = \mu(x)$ for brevity.
Let $P_\mu = I - \frac{\mathbf{1}\mathbf{1}^\top}{d}$ be the projection onto the orthogonal complement of $\mathbf{1}$, and let $\hat{x} = (x - \mu(x)\mathbf{1})/\sigma(x)$ be the normalized centered vector. Then:
\[
J_{\LN}(x) = \frac{1}{\sigma(x)} \Diag(\gamma) \left(P_\mu - \frac{\hat{x}\hat{x}^\top}{d}\right).
\]
The matrix $P_\mu$ is an orthogonal projection with $\|P_\mu\|_2 = 1$. The matrix $\hat{x}\hat{x}^\top/d$ has spectral norm $\|\hat{x}\|_2^2/d \le 1$ (from the proof of Lemma~\ref{lem:ln-output}). 

The matrix $P_\mu - \hat{x}\hat{x}^\top/d$ can be analyzed as follows. Since $\hat{x}$ lies in the range of $P_\mu$ (it has zero mean), we can write this as $P_\mu(I - \hat{x}\hat{x}^\top/d)P_\mu$ restricted to the subspace orthogonal to $\mathbf{1}$. The term $I - \hat{x}\hat{x}^\top/d$ is a rank-one perturbation of the identity. Its eigenvalues are $1$ (with multiplicity $d-1$) and $1 - \|\hat{x}\|_2^2/d \ge 0$. Therefore $\|P_\mu - \hat{x}\hat{x}^\top/d\|_2 \le 1$, so
\[
\|J_{\LN}(x)\|_2 \le \frac{\|\gamma\|_\infty}{\sigma(x)} \le \frac{\|\gamma\|_\infty}{\sqrt{\epsilon}},
\]
since $\sigma(x) \ge \sqrt{\epsilon}$ by definition.
\end{proof}

The next lemma makes precise the contractive structure of the LayerNorm Jacobian used in Section~\ref{sec:layers}; because the relevant bounds are spectral-norm bounds, they also apply to the transposed products used in backpropagation. It strengthens Lemma~\ref{lem:ln-lip} by separating the action on the constant direction from the action on the centered subspace and by giving a quantitative bound on products.

\begin{lemma}[LayerNorm Jacobian: projection structure and product contraction]
\label{lem:ln-projection}
Fix $y\in\R^d$ and let $J_{\LN}(y)$ denote the LayerNorm Jacobian at $y$, with associated mean $\mu(y)$, standard deviation $\sigma(y)\ge\sqrt{\epsilon}$, normalized vector $\hat{y}=(y-\mu(y)\mathbf{1})/\sigma(y)$, and orthogonal projector $P_\mu = I - \tfrac{1}{d}\mathbf{1}\mathbf{1}^\top$ onto the orthogonal complement of $\mathbf{1}$. Then:
\begin{enumerate}[leftmargin=*,itemsep=2pt]
\item[(i)] (Annihilation of the constant direction.) $J_{\LN}(y)\mathbf{1}=0$. Equivalently, the constant subspace $\mathrm{span}\{\mathbf{1}\}$ lies in the kernel of $J_{\LN}(y)$.
\item[(ii)] (Bounded action on the centered subspace.) For every $v\in\R^d$ with $P_\mu v = v$ (i.e.\ $\sum_j v_j=0$),
\[
\|J_{\LN}(y)\,v\|_2 \;\le\; \frac{\|\gamma\|_\infty}{\sigma(y)}\,\|v\|_2.
\]
\item[(iii)] (Additional contraction along $\hat{y}$.) Define $\rho(y):=1-\|\hat{y}\|_2^2/d \in [0,1]$. Then
\[
\|J_{\LN}(y)\,\hat{y}\|_2 \;\le\; \frac{\|\gamma\|_\infty}{\sigma(y)}\,\rho(y)\,\|\hat{y}\|_2,
\]
i.e.\ on the direction $\hat{y}$ the part-(ii) bound is tightened by an extra factor $\rho(y)$, with $\rho(y)<1$ whenever $\hat{y}\neq 0$. Whether $\hat{y}$ is contractive in absolute terms depends on the product $\rho(y)\,\|\gamma\|_\infty/\sigma(y)$.
\item[(iv)] (Geometric decay of products.) For a sequence of LayerNorm Jacobians $J_{\LN,k}$ with scale vectors $\gamma_k$, stability floors $\epsilon_k$, and standard deviations $\sigma_k(y_k)$, evaluated at token inputs $y_\ell,y_{\ell+1},\ldots,y_{m-1}$, and any $u\in\R^d$,
\[
\Bigl\|\prod_{k=\ell}^{m-1}J_{\LN,k}(y_k)\,u\Bigr\|_2 \;\le\; \prod_{k=\ell}^{m-1}\frac{\|\gamma_k\|_\infty}{\sigma_k(y_k)}\;\|u\|_2.
\]
If in addition there exists $\eta\in(0,1)$ such that $\|\gamma_k\|_\infty/\sigma_k(y_k)\le 1-\eta$ at every $k$, then the right-hand side is at most $(1-\eta)^{m-\ell}\|u\|_2$, i.e.\ the product decays geometrically in depth. (Whether this hypothesis holds in practice is an empirical question that our experiments do not test; see the discussion after the proof.)
\end{enumerate}
\end{lemma}

\begin{proof}
Write $J_{\LN}(y)=\frac{1}{\sigma(y)}\Diag(\gamma)\bigl(P_\mu - \hat{y}\hat{y}^\top/d\bigr)$ as in the proof of Lemma~\ref{lem:ln-lip}.

\emph{Part (i).} The factor $P_\mu$ satisfies $P_\mu\mathbf{1}=0$, and $\hat{y}\hat{y}^\top/d$ applied to $\mathbf{1}$ gives $\hat{y}(\hat{y}^\top\mathbf{1})/d=0$ since $\sum_j\hat{y}_j=0$ by construction. Hence $J_{\LN}(y)\mathbf{1}=0$.

\emph{Part (ii).} For $v$ with $P_\mu v = v$, $\hat{y}\hat{y}^\top v/d = (\hat{y}^\top v/d)\,\hat{y}$, so
\[
J_{\LN}(y)v \;=\; \frac{1}{\sigma(y)}\,\Diag(\gamma)\,\Bigl(v - \alpha\hat{y}\Bigr), \qquad \alpha := \frac{\hat{y}^\top v}{d}.
\]
By direct expansion, $\|v-\alpha\hat{y}\|_2^2 = \|v\|_2^2 - 2\alpha\,\hat{y}^\top v + \alpha^2\|\hat{y}\|_2^2 = \|v\|_2^2 - 2d\alpha^2 + \alpha^2\|\hat{y}\|_2^2 = \|v\|_2^2 - \alpha^2(2d-\|\hat{y}\|_2^2)$. Since $\|\hat{y}\|_2^2\le d$ (Lemma~\ref{lem:ln-output} proof), $2d-\|\hat{y}\|_2^2\ge d>0$, so $\|v-\alpha\hat{y}\|_2\le \|v\|_2$. Multiplication by $\Diag(\gamma)$ adds a factor of at most $\|\gamma\|_\infty$ in $\ell_2$ norm, and division by $\sigma(y)$ gives the stated bound.

\emph{Part (iii).} Specialize (ii) to $v=\hat{y}$. Then $\hat{y}^\top\hat{y}/d=\|\hat{y}\|_2^2/d$, so
\[
J_{\LN}(y)\hat{y} = \frac{1}{\sigma(y)}\Diag(\gamma)\Bigl(1-\frac{\|\hat{y}\|_2^2}{d}\Bigr)\hat{y} = \frac{\rho(y)}{\sigma(y)}\Diag(\gamma)\hat{y},
\]
giving the bound in (iii). With standard LayerNorm $\sigma(y)^2 = \tfrac{1}{d}\|y-\mu(y)\mathbf{1}\|_2^2 + \epsilon$ and $\epsilon>0$, $\|\hat{y}\|_2^2 = d(\sigma(y)^2-\epsilon)/\sigma(y)^2 \le d$, so $\rho(y) = \epsilon/\sigma(y)^2 \in (0,1]$, with $\rho(y)<1$ whenever $y$ is non-constant.

\emph{Part (iv).} As shown in the proof of Lemma~\ref{lem:ln-lip}, each module satisfies $\|J_{\LN,k}(y_k)\|_2 \le \|\gamma_k\|_\infty/\sigma_k(y_k)$ as an operator-norm bound on all of $\R^d$. Submultiplicativity of operator norms then yields $\|\prod_{k=\ell}^{m-1}J_{\LN,k}(y_k)\,u\|_2 \le \prod_{k=\ell}^{m-1}(\|\gamma_k\|_\infty/\sigma_k(y_k))\|u\|_2$ for any $u\in\R^d$. The geometric-decay statement is the immediate consequence of the per-step hypothesis $\|\gamma_k\|_\infty/\sigma_k(y_k)\le 1-\eta$.
\end{proof}

Lemma~\ref{lem:ln-projection} is a statement about a single token vector $y\in\R^d$. The transformer state is a matrix, and Section~\ref{sec:layers} applies the lemma at the sequence level; the following corollary supplies the lift.

\begin{corollary}[Sequence-level LayerNorm product bound]
\label{cor:ln-sequence}
For $Y\in\R^{L\times d}$ with rows $y_1,\ldots,y_L$, let $\mathcal{J}_{\LN,k}(Y)$ denote the Jacobian of the rowwise map $Y\mapsto\LN_k(Y)$; since LayerNorm acts on each token independently, $\mathcal{J}_{\LN,k}(Y)$ is block diagonal with token blocks $J_{\LN,k}(y_i)$. Define the per-token factors
\[
a_{k,i} \;:=\; \frac{\|\gamma_k\|_\infty}{\sigma_k(y_{k,i})},
\]
where $y_{k,i}$ is token $i$ of the layer-$k$ input $Y_k$. Then for states $Y_\ell,\ldots,Y_{m-1}$,
\[
\Bigl\|\prod_{k=\ell}^{m-1}\mathcal{J}_{\LN,k}(Y_k)\Bigr\|_2
\;\le\;
\max_{i\in[L]}\ \prod_{k=\ell}^{m-1} a_{k,i}
\;\le\;
\prod_{k=\ell}^{m-1}\ \max_{i\in[L]} a_{k,i}.
\]
In particular, if $\max_i a_{k,i}\le 1-\eta$ for every $k$ and some $\eta\in(0,1)$, the sequence-level LayerNorm-only product contracts geometrically, $\|\prod_k \mathcal{J}_{\LN,k}(Y_k)\|_2\le(1-\eta)^{m-\ell}$. The same bounds hold for the transposed product used in backpropagation.
\end{corollary}

\begin{proof}
A product of block-diagonal matrices with common block structure is block diagonal with blocks equal to the products of the corresponding token blocks, so $\prod_{k}\mathcal{J}_{\LN,k}(Y_k)$ has token blocks $\prod_{k}J_{\LN,k}(y_{k,i})$. The spectral norm of a block-diagonal matrix is the maximum block spectral norm, and each token block obeys $\|\prod_k J_{\LN,k}(y_{k,i})\|_2\le\prod_k a_{k,i}$ by Lemma~\ref{lem:ln-projection}(iv) applied tokenwise. The second inequality bounds each token's product by the product of per-layer maxima, and transposition preserves spectral norms.
\end{proof}

\paragraph{Application to post-LN gradient flow.}
Theorem~\ref{thm:postln-gradient}(iii) writes the end-to-end Jacobian as $\prod_k \mathcal{J}_{\LN,k}(Y_k)+\bigl(\text{terms involving }J_{\MHA}\bigr)$, and we call the first term the LayerNorm-only term. Corollary~\ref{cor:ln-sequence} shows that this term contracts geometrically at the sequence level whenever the per-layer worst-token factor $\max_i a_{k,i}=\max_i\|\gamma_k\|_\infty/\sigma_k(y_{k,i})$ is bounded away from $1$ at every layer; this is the quantity one would need to measure to test the condition empirically. The corollary controls only the LayerNorm-only term; it does not bound the complete sum of terms containing MHA Jacobians, and therefore does not prove contraction of the full post-LN end-to-end Jacobian. Because the bounds are spectral-norm bounds, they apply equally to the transposed products used in backpropagation.

\section{Proof of Theorem~\ref{thm:softmax}: Exact Softmax Jacobian Norm}
\label{app:softmax}

This appendix provides the proof of Theorem~\ref{thm:softmax} from Section~\ref{sec:softmax}, which establishes the exact $\ell_\infty \to \ell_1$ operator norm of the softmax Jacobian.

\begin{proof}[Proof of Theorem~\ref{thm:softmax}]
Write $J := J_\tau(u) = \frac{1}{\tau}(\Diag(p) - pp^\top)$ with $p = s_\tau(u)$; the form of $J$ follows from the standard softmax Jacobian by the chain rule applied to $u\mapsto u/\tau$. The $\ell_\infty \to \ell_1$ operator norm is $\|J\|_{\infty \to 1} = \max_{\|x\|_\infty \le 1} \|Jx\|_1$. Since $\|Jx\|_1$ is a convex function of $x$ and the $\ell_\infty$ ball is a polytope, the maximum is achieved at a vertex, i.e., at some $x \in \{\pm 1\}^L$. Indeed, any point of the $\ell_\infty$ unit ball is a convex combination of sign vertices; convexity therefore gives $\|Jx\|_1\le\max_{v\in\{\pm1\}^L}\|Jv\|_1$.

For any vertex $x$, let $S = \{i : x_i = +1\}$ denote the positive indices and $\bar{S} = [L] \setminus S$ the negative indices. Define 
\[
\alpha = p^\top x = \sum_{i \in S} p_i - \sum_{i \in \bar{S}} p_i = p(S) - (1 - p(S)) = 2p(S) - 1.
\]

The Jacobian $J = \frac{1}{\tau}(\Diag(p) - pp^\top)$ acts on $x$ as:
\[
(Jx)_i = \frac{1}{\tau}\left(p_i x_i - p_i (p^\top x)\right) = \frac{p_i}{\tau}(x_i - \alpha).
\]

We compute $(Jx)_i$ separately for each case:
\begin{itemize}
    \item For $i \in S$: $x_i = +1$, so $(Jx)_i = \frac{p_i}{\tau}(1 - \alpha) = \frac{p_i}{\tau}(1 - (2p(S) - 1)) = \frac{2p_i}{\tau}(1 - p(S))$.
    \item For $i \in \bar{S}$: $x_i = -1$, so $(Jx)_i = \frac{p_i}{\tau}(-1 - \alpha) = \frac{p_i}{\tau}(-1 - (2p(S) - 1)) = \frac{-2p_i}{\tau}p(S)$.
\end{itemize}

Since $p(S) \in [0,1]$, we have $\alpha = 2p(S) - 1 \in [-1, 1]$. This ensures $(1 - \alpha) \ge 0$ and $(1 + \alpha) \ge 0$, so the signs are determined by the subset membership. Computing the $\ell_1$ norm:
\begin{align*}
\|Jx\|_1 &= \sum_{i \in S} \left|\frac{2p_i}{\tau}(1-p(S))\right| + \sum_{i \in \bar{S}} \left|\frac{-2p_i}{\tau}p(S)\right| \\
&= \frac{2(1-p(S))}{\tau} \sum_{i \in S} p_i + \frac{2p(S)}{\tau} \sum_{i \in \bar{S}} p_i \\
&= \frac{2(1-p(S))}{\tau} \cdot p(S) + \frac{2p(S)}{\tau} \cdot (1 - p(S)) \\
&= \frac{4}{\tau} p(S)(1-p(S)).
\end{align*}

Since vertices $x \in \{\pm 1\}^L$ correspond bijectively to subsets $S \subseteq [L]$, maximizing over all vertices is equivalent to maximizing over all subsets:
\begin{align*}
\|J\|_{\infty \to 1} &= \max_{S \subseteq [L]} \frac{4}{\tau} p(S)(1-p(S)) \\
&= \frac{4}{\tau} \max_{S \subseteq [L]} p(S)(1-p(S)) = \frac{\theta(p)}{\tau}.
\end{align*}

\paragraph{Attainment.} Let $S^*$ be any subset achieving the maximum in the definition of $\theta(p)$, and define $x^* \in \{\pm 1\}^L$ by $x^*_i = +1$ if $i \in S^*$ and $x^*_i = -1$ otherwise. Then $\|x^*\|_\infty = 1$ and
\[
\|Jx^*\|_1 = \frac{4}{\tau} p(S^*)(1-p(S^*)) = \frac{\theta(p)}{\tau}.
\]
Thus the maximum is attained by $x^*$.
\end{proof}

\begin{remark}[Connection to total variation]
Let $p_\varepsilon=\softmax((u+\varepsilon x)/\tau)$. Then
\[
\operatorname{TV}(p_\varepsilon,p) \;=\; \frac{|\varepsilon|}{2}\,\|Jx\|_1 + o(|\varepsilon|).
\]
Thus $\|Jx\|_1$ is twice the first-order rate of probability-mass redistribution under the logit perturbation $x$, and the worst-case perturbation $x^*$ identifies the optimal bisection of probability mass.
\end{remark}

\begin{remark}[Comparison to same-norm bounds]
The softmax Jacobian $\Diag(p)-pp^\top$ and its basic Lipschitz properties are classical \citep{gao2017properties}. Distribution-independent same-norm bounds have since been sharpened, culminating in the tight constant $1/(2\tau)$ for the tempered softmax uniformly across all $\ell_p\to\ell_p$ norms \citep{nair2025softmax}. Such same-norm statements and the present cross-norm identity concern different induced norms and are not directly ordered. Our contribution is an exact, distribution-dependent characterization in the cross-norm $\ell_\infty\to\ell_1$ geometry, whereas the cited results are distribution-independent global constants in same-norm geometries.
\end{remark}
\begin{proof}[Proof of Corollary~\ref{cor:regimes}]
We verify each case by computing $\theta(p) = 4\max_{S \subseteq [L]} p(S)(1-p(S))$.

\textbf{Case 1: Uniform distribution.} Let $p = \frac{1}{L}\mathbf{1}$. For any subset $S$ with $|S| = k$, we have $p(S) = k/L$. The function $g(t) = t(1-t)$ is maximized at $t = 1/2$. 

If $L$ is even, choosing $|S| = L/2$ gives $p(S) = 1/2$, so $\theta(p) = 4 \cdot \frac{1}{2} \cdot \frac{1}{2} = 1$.

If $L$ is odd, the attainable subset masses are $k/L$ for integer $k$, and $g(k/L)$ is maximized over integers at $k = \lfloor L/2 \rfloor$ (equivalently $k=\lceil L/2\rceil$), since $g$ increases on $[0,1/2]$ and decreases on $[1/2,1]$. Hence
\[
\theta(p) = 4 \cdot \frac{\lfloor L/2 \rfloor}{L} \cdot \frac{\lceil L/2 \rceil}{L} = 1 - 1/L^2 \quad\text{for odd } L.
\]
In particular, $\theta(p) = 1$ for even $L$ and $\theta(p) = 1 - 1/L^2$ for odd $L$; the uniform case is exact at $\theta=1$ only when $L$ is even, and asymptotically maximal otherwise.

\textbf{Case 2: One-hot distribution.} Let $p = e_i$ (all mass on token $i$). For any subset $S$:
\[
p(S) = \begin{cases} 1 & \text{if } i \in S \\ 0 & \text{if } i \notin S \end{cases}
\]
Thus $p(S)(1-p(S)) = 1 \cdot 0 = 0$ or $0 \cdot 1 = 0$ for all $S$, giving $\theta(p) = 0$.

\textbf{Case 3: Dominant atom.} Suppose the top token has mass $p_1 = 1-\kappa$ with $\kappa\le 1/2$, and the remaining mass $\kappa$ is distributed among other tokens. For any subset $S$:
\begin{itemize}
    \item If $1 \in S$: $p(S) \ge 1-\kappa$, so $p(S)(1-p(S)) \le (1-\kappa)\kappa$.
    \item If $1 \notin S$: $p(S) \le \kappa$, so $p(S)(1-p(S)) \le \kappa(1-\kappa)$.
\end{itemize}
In both cases, $p(S)(1-p(S)) \le \kappa(1-\kappa)$, giving $\theta(p) \le 4\kappa(1-\kappa)$. The bound is achieved by taking $S$ to be the set of all non-top tokens, so $p(S)=\kappa$ and therefore $\theta(p)=4\kappa(1-\kappa)$ exactly.

\textbf{Case 4: Top-$k$ uniform distribution.} Suppose $p_i = 1/k$ for $i \in \{1, \ldots, k\}$ and $p_i = 0$ otherwise. Any subset $S$ satisfies $p(S) = |S \cap \{1,\ldots,k\}|/k$. Let $m = |S \cap \{1,\ldots,k\}| \in \{0, 1, \ldots, k\}$.

We maximize $g(m/k) = (m/k)(1 - m/k)$ over integer $m \in \{0, \ldots, k\}$. The continuous maximum of $t(1-t)$ is at $t = 1/2$. The closest integer solutions are $m = \lfloor k/2 \rfloor$ and $m = \lceil k/2 \rceil$.

Taking $m^* = \lfloor k/2 \rfloor$:
\[
p(S^*) = \frac{\lfloor k/2 \rfloor}{k}, \quad 1 - p(S^*) = \frac{k - \lfloor k/2 \rfloor}{k} = \frac{\lceil k/2 \rceil}{k}.
\]

The maximum value is:
\[
\max_S p(S)(1-p(S)) = \frac{\lfloor k/2 \rfloor}{k} \cdot \frac{\lceil k/2 \rceil}{k}.
\]

For even $k$: $\lfloor k/2 \rfloor = \lceil k/2 \rceil = k/2$, so the maximum is $(1/2)(1/2) = 1/4$, giving $\theta(p) = 1$.

For odd $k$: $\lfloor k/2 \rfloor = (k-1)/2$ and $\lceil k/2 \rceil = (k+1)/2$, so the maximum is $\frac{(k-1)/2}{k} \cdot \frac{(k+1)/2}{k} = \frac{k^2-1}{4k^2}$.

To express this uniformly, note that the partition gap is $\delta = |p(S^*) - 1/2| = |\lfloor k/2 \rfloor/k - 1/2|$. For even $k$, $\delta = 0$. For odd $k$, $\delta = 1/(2k)$.

Using $\theta(p) = 1 - 4\delta^2$ (the equivalent characterization):
\[
\theta(p) = 1 - 4\left(\frac{1}{2} - \frac{\lfloor k/2 \rfloor}{k}\right)^2 = 1 - \left(1 - \frac{2\lfloor k/2 \rfloor}{k}\right)^2.
\]

\textbf{Case 5: Maximum-weight lower bound.} Order the probabilities arbitrarily and let $c_k$ be the first cumulative sum to cross $1/2$: $c_{k-1} < 1/2 \le c_k = c_{k-1} + p_k$. The interval $[c_{k-1}, c_k]$ contains $1/2$ and has length $p_k \le p_{\max}$, so at least one of $c_{k-1}$ and $c_k$ is within $p_k/2 \le p_{\max}/2$ of $1/2$. Hence there exists a subset mass $q$ (a prefix sum of the chosen ordering) satisfying $|2q-1| \le p_{\max}$, and therefore
\[
\theta(p) \;\ge\; 4q(1-q) \;=\; 1-(2q-1)^2 \;\ge\; 1 - p_{\max}^2.
\]
\end{proof}
\section{Proofs for Section~\ref{sec:layers}: Layerwise Stability}
\label{app:layers}

This appendix provides proofs for the layerwise stability results in Section~\ref{sec:layers}. We organize the proofs to follow the main text: first the MHA local Jacobian bound and its bounded-domain corollary, then per-layer Lipschitz bounds for both architectures, and finally gradient flow analysis.

\subsection{Proof of Theorem~\ref{thm:mha}: MHA Local Jacobian Bound}

\begin{proof}
We bound the induced norm of the differential $D\MHA(U)$ by decomposing it into value and attention-weight pathway contributions. Throughout, $H$ denotes a tangent direction in $\R^{L\times d}$, and we use the unnormalized row norm $\|Z\|_{\infty,2}=\max_i\|z_i\|_2$ of Section~\ref{sec:framework} for the $d_h$-width head matrices, together with the transfer inequality $\|XW\|_{\infty,2} \le \|X\|_{\infty,\mathrm{rms}}\sqrt{d}\,\|W\|_2$ for $X\in\R^{L\times d}$.

For head $h$, the differential of $O_h(U) = A_h(U)\,V_h(U)$ with $V_h(U)=UW_h^V$ is
\[
D O_h(U)[H] \;=\; A_h(U)\,(H W_h^V) \;+\; D A_h(U)[H]\; V_h(U).
\]

\paragraph{Value pathway.} By Lemma~\ref{lem:stochastic} applied at head width, $\|A_h Z\|_{\infty,2} \le \|Z\|_{\infty,2}$ for row-stochastic $A_h$. Hence
\[
\|A_h(U)(H W_h^V)\|_{\infty,2} \le \|H W_h^V\|_{\infty,2} \le \|H\|_{\infty,\mathrm{rms}}\sqrt{d}\,\|W_h^V\|_2.
\]

\paragraph{Attention-weight pathway.} The scores are $S_h(U) = (UW_h^Q)(UW_h^K)^\top/\sqrt{d_h}$. For a tangent direction $H$, the differential of $S_h$ is
\[
D S_h(U)[H] = \frac{1}{\sqrt{d_h}}\bigl[(H W_h^Q)(UW_h^K)^\top + (UW_h^Q)(H W_h^K)^\top\bigr].
\]
Entrywise,
\begin{align*}
|(D S_h(U)[H])_{ij}| &\le \frac{1}{\sqrt{d_h}}\bigl(\|H W_h^Q\|_{\infty,2} \|UW_h^K\|_{\infty,2} + \|UW_h^Q\|_{\infty,2} \|H W_h^K\|_{\infty,2}\bigr) \\
&\le \frac{2d}{\sqrt{d_h}} \|W_h^Q\|_2 \|W_h^K\|_2 \|H\|_{\infty,\mathrm{rms}} \|U\|_{\infty,\mathrm{rms}},
\end{align*}
so the same bound holds for $\|D S_h(U)[H]\|_{\infty,\infty}$ (maximum absolute entry).

Row $i$ of $A_h(U)$ is the tempered softmax $s_\tau$ of row $i$ of $S_h(U)$. By Theorem~\ref{thm:softmax}, the row Jacobian satisfies $\|J_\tau(S_h[i,:])\|_{\infty\to 1} = \theta(A_h[i,:])/\tau$, hence by the chain rule
\[
\|D A_h(U)[H][i,:]\|_1 \le \frac{\theta(A_h[i,:])}{\tau}\,\|D S_h(U)[H][i,:]\|_\infty \le \frac{\tilde{\theta}_h(U)}{\tau}\,\|D S_h(U)[H]\|_{\infty,\infty}.
\]
For the action on values, each row of $(D A_h[H])\,V_h$ is $\sum_j (D A_h[H])_{ij}\,(V_h)_j$, so
\[
\|D A_h(U)[H]\,V_h\|_{\infty,2} \le \|D A_h(U)[H]\|_{\infty,1}\,\|V_h\|_{\infty,2},
\]
with $\|D A_h(U)[H]\|_{\infty,1} = \max_i \|D A_h(U)[H][i,:]\|_1$ and $\|V_h\|_{\infty,2} \le \|U\|_{\infty,\mathrm{rms}}\sqrt{d}\,\|W_h^V\|_2$. Combining the three displays,
\[
\|D A_h(U)[H]\,V_h\|_{\infty,2} \;\le\; \frac{\tilde{\theta}_h(U)}{\tau}\cdot\frac{2d\,\|U\|_{\infty,\mathrm{rms}}^2}{\sqrt{d_h}}\,\|W_h^Q\|_2\|W_h^K\|_2\|W_h^V\|_2\cdot \|H\|_{\infty,\mathrm{rms}}\sqrt{d}.
\]
Writing $\bar{B}_U = \|U\|_{\infty,\mathrm{rms}}\sqrt{d}$, the right-hand side equals $\frac{\tilde{\theta}_h(U)}{\tau}\cdot\frac{2\bar{B}_U^2}{\sqrt{d_h}}\,\|W_h^Q\|_2\|W_h^K\|_2\|W_h^V\|_2\cdot\|H\|_{\infty,\mathrm{rms}}\sqrt{d}$.

\paragraph{Concatenation and output projection.} For head-output perturbations $Z_h\in\R^{L\times d_h}$,
\[
\|[Z_1\|\cdots\|Z_H]\|_{\infty,2} \;\le\; \sum_{h=1}^H \|Z_h\|_{\infty,2},
\]
since each row of the concatenation has $\ell_2$ norm $\bigl(\sum_h\|z_{h,i}\|_2^2\bigr)^{1/2}\le\sum_h\|z_{h,i}\|_2$. Applying $W^O$ (a factor $\|W^O\|_2$ on each row) and dividing by $\sqrt{d}$ to return to the block-$\infty$/RMS norm gives
\[
\|D\MHA(U)[H]\|_{\infty,\mathrm{rms}} \le \|W^O\|_2 \sum_{h=1}^H \left(\|W_h^V\|_2 + \frac{\tilde{\theta}_h(U)}{\tau} \cdot \frac{2\bar{B}_U^2}{\sqrt{d_h}}\|W_h^Q\|_2\|W_h^K\|_2\|W_h^V\|_2\right)\|H\|_{\infty,\mathrm{rms}},
\]
which is the stated bound on $\Lambda_{\MHA}(U)$. The bound depends on weight norms, $\bar{B}_U$, $d$, $d_h$, $\tau$, and $\tilde{\theta}_h(U) \le 1$. It carries no explicit sequence-length factor because the block-$\infty$/RMS geometry measures worst-case token magnitude rather than aggregating across tokens.
\end{proof}

\begin{proof}[Proof of Corollary~\ref{cor:mha-bounded}]
The set $\mathcal{D}_B$ is convex. For $U, U' \in \mathcal{D}_B$, define $U_t = U + t(U'-U)$ for $t\in[0,1]$. Then
\[
\MHA(U') - \MHA(U) = \int_0^1 D\MHA(U_t)[U'-U]\,dt.
\]
Every $U_t$ lies in $\mathcal{D}_B$, so Theorem~\ref{thm:mha} applies at each $U_t$ with $\bar{B}_{U_t}^2 \le dB^2$, and $\tilde{\theta}_h(U_t) \le 1$ since $\theta$ takes values in $[0,1]$. Taking block-$\infty$/RMS norms inside the integral yields the stated uniform constant $L_{\MHA}^{(B)}$.
\end{proof}

\subsection{Proof of Theorem~\ref{thm:preln-full}: Pre-LN Full Layer Bound}

\begin{proof}
A complete pre-LN transformer layer consists of two sublayers:
\begin{align*}
Z_\ell &= X_\ell + \MHA(\LN(X_\ell)), \\
X_{\ell+1} &= Z_\ell + \FFN(\LN(Z_\ell)).
\end{align*}

\paragraph{Bounded input magnitude.} In pre-LN, the MHA sublayer receives input $U = \LN(X_\ell)$. By Lemma~\ref{lem:ln-output}, for any input $X_\ell$, regardless of $\|X_\ell\|_{\infty,\mathrm{rms}}$,
\[
\|\LN(X_\ell)\|_{\infty,\mathrm{rms}} \le \|\gamma\|_\infty + \|\beta\|_\infty \le B_{\LN},
\]
so the MHA input lies in the bounded domain $\mathcal{D}_{B_{\LN}}$ of Corollary~\ref{cor:mha-bounded}, independently of layer index $\ell$, total depth $N$, sequence length $L$, and residual stream magnitude $\|X_\ell\|$.

\paragraph{MHA sublayer.} On this domain, the map $X_\ell \mapsto Z_\ell = X_\ell + \MHA(\LN(X_\ell))$ has Lipschitz constant:
\[
L_{\mathrm{MHA-sub}} \le 1 + \operatorname{Lip}(\MHA;\mathcal{D}_{B_{\LN}}) \cdot \mathrm{Lip}(\LN) \le 1 + L_{\MHA,\ell}^{(B_{\LN})} \cdot L_{\LN}.
\]

\paragraph{FFN sublayer.} The map $Z_\ell \mapsto X_{\ell+1} = Z_\ell + \FFN(\LN(Z_\ell))$ has Lipschitz constant:
\[
L_{\mathrm{FFN-sub}} \le 1 + L_{\FFN,\ell} \cdot L_{\LN}.
\]

For a standard two-layer network $\FFN(x) = W_2 \sigma(W_1 x + b_1) + b_2$, $L_{\FFN,\ell}$ is bounded by $\|W_2\|_2\,\mathrm{Lip}(\sigma)\,\|W_1\|_2$; for $1$-Lipschitz activations (e.g.,\ ReLU) this reduces to $\|W_2\|_2\|W_1\|_2$.

\paragraph{Full layer.} The composition has Lipschitz constant:
\begin{align*}
L_{\mathrm{layer},\ell}^{\mathrm{pre}} &\le L_{\mathrm{MHA-sub}} \cdot L_{\mathrm{FFN-sub}} \\
&= (1 + L_{\LN} \cdot L_{\MHA,\ell}^{(B_{\LN})})(1 + L_{\LN} \cdot L_{\FFN,\ell}).
\end{align*}

The constant $L_{\MHA,\ell}^{(B_{\LN})}$ depends on the layer-$\ell$ weight norms, $B_{\LN}$, $d$, $d_h$, and $\tau$, but carries no explicit dependence on $N$ or $L$; by Lemma~\ref{lem:ln-lip}, each LayerNorm module's Lipschitz constant is bounded by $L_{\LN}$. If the relevant weight norms and LayerNorm parameters are uniformly bounded across layers, the right-hand side is uniform in $\ell$.
\end{proof}

\subsection{Proof of Theorem~\ref{thm:postln-full}: Post-LN Full Layer Bound}

\begin{proof}
A complete post-LN transformer layer consists of two sublayers:
\begin{align*}
Z_\ell &= \LN(X_\ell + \MHA(X_\ell)), \\
X_{\ell+1} &= \LN(Z_\ell + \FFN(Z_\ell)).
\end{align*}

\paragraph{Input magnitude analysis.} In post-LN, for $\ell \ge 1$ the input $X_\ell$ to layer $\ell$ is the output of the LayerNorm from layer $\ell - 1$. By Lemma~\ref{lem:ln-output}:
\[
\|X_\ell\|_{\infty,\mathrm{rms}} \le \|\gamma\|_\infty + \|\beta\|_\infty \le B_{\LN}.
\]
Thus the MHA input lies in $\mathcal{D}_{B_{\LN}}$ and Corollary~\ref{cor:mha-bounded} applies with constant $L_{\MHA,\ell}^{(B_{\LN})}$. The first layer ($\ell=0$) requires a separate bound on its input magnitude; the stated theorem restricts to $\ell\ge1$.

\paragraph{MHA sublayer.} On this domain, the map $X_\ell \mapsto Z_\ell = \LN(X_\ell + \MHA(X_\ell))$ has Lipschitz constant:
\[
L_{\mathrm{MHA-sub}} \le L_{\LN} \cdot (1 + L_{\MHA,\ell}^{(B_{\LN})}).
\]

\paragraph{FFN sublayer.} The map $Z_\ell \mapsto X_{\ell+1} = \LN(Z_\ell + \FFN(Z_\ell))$ has Lipschitz constant:
\[
L_{\mathrm{FFN-sub}} \le L_{\LN} \cdot (1 + L_{\FFN,\ell}).
\]

\paragraph{Full layer.} The composition has Lipschitz constant:
\begin{align*}
L_{\mathrm{layer},\ell}^{\mathrm{post}} &\le L_{\mathrm{MHA-sub}} \cdot L_{\mathrm{FFN-sub}} \\
&= L_{\LN}^2 (1 + L_{\MHA,\ell}^{(B_{\LN})})(1 + L_{\FFN,\ell}).
\end{align*}

This depends on the layer-$\ell$ weight norms, LayerNorm parameters, and architectural constants, with no explicit dependence on $N$ or $L$.
\end{proof}

\subsection{Proof of Theorem~\ref{thm:preln-gradient}: Pre-LN Identity Gradient Path}

\begin{proof}
We prove each part in sequence.

\paragraph{Part (i): Bounded sublayer Jacobian.}
The pre-LN layer computes $X_{\ell+1} = X_\ell + \MHA(\LN(X_\ell))$. By the chain rule, the layer Jacobian is $I + J_k$ where $J_k := J_{\MHA}(\LN(X_k)) \cdot J_{\LN}(X_k)$.

By Lemma~\ref{lem:ln-output}, $\LN(X_k)\in\mathcal{D}_{B_{\LN}}$, so Theorem~\ref{thm:mha} with $\tilde\theta_h\le1$ bounds $\|J_{\MHA}(\LN(X_k))\|$ by $L_{\MHA,k}^{(B_{\LN})}$. By Lemma~\ref{lem:ln-lip}, $\|J_{\LN}(X_k)\| \le L_{\LN}$. Composing under the block-$\infty$/RMS norm,
\[
\|J_k\|_{\infty,\mathrm{rms}\to\infty,\mathrm{rms}} \;\le\; L_{\MHA,k}^{(B_{\LN})} \cdot L_{\LN} \;=:\; C_k,
\]
with no explicit dependence on $N$ or $L$.

\paragraph{Part (ii): End-to-end Jacobian expansion.}
For $\ell < m$:
\[
\frac{\partial X_m}{\partial X_\ell} = \prod_{k=\ell}^{m-1} \frac{\partial X_{k+1}}{\partial X_k} = \prod_{k=\ell}^{m-1} (I + J_k).
\]
Expanding:
\begin{align*}
\prod_{k=\ell}^{m-1} (I + J_k) &= I + \sum_{k=\ell}^{m-1} J_k + \sum_{\ell \le k_1 < k_2 < m} J_{k_2} J_{k_1} + \cdots \\
&= I + \sum_{k=\ell}^{m-1} J_k + R_{\ge 2},
\end{align*}
where $R_{\ge 2}$ is the exact sum of all ordered products containing at least two sublayer Jacobians. The identity matrix $I$ appears as an additive term.

\paragraph{Part (iii): Structural read-off.}
Substituting $g$ as the gradient signal at layer $m$ into the expansion of part (ii): $g^\top \partial X_m/\partial X_\ell = g^\top + g^\top\sum_{k=\ell}^{m-1} J_k + g^\top R_{\ge 2}$. The first term is $g$ itself, transmitted through the residual connections without multiplication by any sublayer Jacobian; the second is a sum of contributions from each $J_k$ taken individually. This identifies a distinguished identity path in the expansion. It is a structural statement, not a norm lower bound: the terms $g^\top\sum_k J_k + g^\top R_{\ge 2}$ may cancel part or all of the identity contribution, as in the alternating-sign scalar example following Theorem~\ref{thm:preln-gradient}. Contrast with post-LN (Theorem~\ref{thm:postln-gradient}), where the expansion contains no bare additive identity term and the LayerNorm-only term is subject to the sufficient contraction condition of Corollary~\ref{cor:ln-sequence}.
\end{proof}

\subsection{Proof of Theorem~\ref{thm:postln-gradient}: Post-LN No Identity Gradient Path}

\begin{proof}
We prove each part in sequence.

\paragraph{Part (i): End-to-end Jacobian formula.}
The post-LN layer computes $X_{\ell+1} = \LN(Y_\ell)$ where $Y_\ell = X_\ell + \MHA(X_\ell)$. By the chain rule:
\[
\frac{\partial X_{\ell+1}}{\partial X_\ell} = J_{\LN}(Y_\ell) \cdot \frac{\partial Y_\ell}{\partial X_\ell} = J_{\LN}(Y_\ell) \cdot (I + J_{\MHA}(X_\ell)).
\]
For $\ell < m$, chaining these derivatives:
\[
\frac{\partial X_m}{\partial X_\ell} = \prod_{k=\ell}^{m-1} J_{\LN}(Y_k) (I + J_{\MHA}(X_k)).
\]

\paragraph{Part (ii): LayerNorm Jacobian is projection-like.}
The LayerNorm Jacobian at $y \in \R^d$ is:
\[
J_{\LN}(y) = \frac{\Diag(\gamma)}{\sigma(y)}\left(I - \frac{\mathbf{1}\mathbf{1}^\top}{d} - \frac{(y-\mu\mathbf{1})(y-\mu\mathbf{1})^\top}{d\sigma(y)^2}\right),
\]
where $\mu = \frac{1}{d}\sum_j y_j$ and $\sigma^2 = \frac{1}{d}\sum_j(y_j - \mu)^2 + \epsilon$.

This matrix has two key properties:
\begin{itemize}
\item It annihilates the constant direction: $J_{\LN}(y)\mathbf{1} = 0$ (from $I - \frac{\mathbf{1}\mathbf{1}^\top}{d}$).
\item It is bounded on the centered subspace and has an extra $\rho(y)$ shrinkage factor when the input direction is $\hat y$ (Lemma~\ref{lem:ln-projection}).
\end{itemize}
Thus $J_{\LN}(y) \neq I + \text{small perturbation}$; it is projection-like rather than identity-like.

\paragraph{Part (iii): No additive identity term.}
Expanding the product over the choices $\{I, J_{\MHA}(X_k)\}$ yields $2^{m-\ell}$ ordered terms. Every term contains one LayerNorm Jacobian from every residual sublayer, in the original order. The unique term obtained by choosing $I$ from every bracket is
\[
\mathcal{J}_{\LN,m-1}(Y_{m-1})\cdots \mathcal{J}_{\LN,\ell}(Y_\ell),
\]
which we call the LayerNorm-only term. The expansion contains no bare additive identity term. Contrast this with pre-LN, where $\prod_{k=\ell}^{m-1}(I + J_k) = I + \sum_k J_k + R_{\ge 2}$ has $I$ as an additive term. The LayerNorm-only term is distinguished by containing zero MHA-Jacobian factors, not by dominating the remaining terms.

The LayerNorm-only term is governed at the sequence level by Corollary~\ref{cor:ln-sequence}: writing $a_{k,i}=\|\gamma_k\|_\infty/\sigma_k(y_{k,i})$ for token $i$ of the layer-$k$ input,
\[
\Bigl\|\prod_{k=\ell}^{m-1} \mathcal{J}_{\LN,k}(Y_k)\Bigr\|_2 \;\le\; \max_{i\in[L]}\prod_{k=\ell}^{m-1} a_{k,i} \;\le\; \prod_{k=\ell}^{m-1}\max_{i\in[L]} a_{k,i},
\]
which decays geometrically whenever each per-layer worst-token factor is bounded away from $1$. This controls only the LayerNorm-only term, not the full expansion.
\end{proof}
\section{Experimental Details and Additional Results}
\label{app:experiments}

This appendix provides full experimental details and additional results supporting Section~\ref{sec:experiments}.

\subsection{Experimental Setup}

\paragraph{Model architecture.}
We train GPT-2-style transformers with the following configuration:
\begin{itemize}[leftmargin=*,itemsep=2pt]
    \item Model dimension: $d = 1280$
    \item Number of layers: $N = 36$
    \item Attention heads: $H = 20$
    \item Head dimension: $d_h = 64$
    \item FFN hidden dimension: $d_{\text{ff}} = 5120$
    \item Context length: $L = 1024$
    \item Total parameters: 774M
    \item FFN activation: GELU
    \item Tokenizer: GPT-2 byte-level BPE, vocabulary $50{,}257$
    \item Dropout: $0.1$ on attention probabilities, the FFN output, and the embedding sum
    \item Linear biases: disabled in all projections (attention $W^Q,W^K,W^V,W^O$, both FFN layers, and the output head); only LayerNorm carries affine parameters
    \item Initialization: $\mathcal{N}(0,0.02^2)$ for all linear and embedding weights, with attention-output and FFN-output projections rescaled to std $0.02/\sqrt{2N}$; LayerNorm initialized to $\gamma=1,\beta=0$; input embedding and output head weights tied
\end{itemize}

\paragraph{Training configuration.}
\begin{itemize}[leftmargin=*,itemsep=2pt]
    \item Optimizer: AdamW with $\beta = (0.9, 0.95)$, $\epsilon = 10^{-8}$ (framework default)
    \item Weight decay: 0.1
    \item Learning rate: $3 \times 10^{-4}$ peak
    \item Schedule: linear warmup over $500$ optimizer updates, cosine decay
    \item Batch size: 64 (micro-batch 4, gradient accumulation 16)
    \item Gradient clipping: norm 1.0
    \item Precision: bfloat16
    \item Hardware: NVIDIA H200 GPU
    \item Logged micro-batch iterations: $32{,}000$
    \item Optimizer updates completed: $2000$
    \item Random seeds: 42, 123, 456
\end{itemize}

All training counters in this paper are reported in \emph{optimizer updates} unless a figure axis is explicitly labeled ``micro-batch iteration''; with gradient-accumulation factor $16$, one optimizer update corresponds to $16$ micro-batch iterations.

\paragraph{Measurement cadence and data order.}
Attention statistics are recorded every $10$ optimizer updates (and additionally at the first and last update), on the last micro-batch of the current update; layerwise gradient norms are recorded every $50$ optimizer updates. The FineWeb stream is consumed in a fixed order that does not depend on the torch seed, so all runs (both architectures, all seeds) see the same data in the same order; only initialization and dropout differ across seeds. This is why some spike updates coincide exactly across runs.

\paragraph{Data.}
Models are trained on FineWeb \citep{penedo2024fineweb}, a large-scale web text corpus. We use streaming data loading with packed sequences of length 1024.

\paragraph{Metrics.}
We track the following quantities over the logged trajectory:
\begin{itemize}[leftmargin=*,itemsep=2pt]
    \item Training loss (cross-entropy)
    \item Gradient norm (before clipping)
    \item Lower-bound certificate $\thetacert(p)$ computed by the greedy prefix procedure on sampled attention distributions, with exact enumeration used only in the small-$L$ exhaustive-check setting
    \item Row maximum attention weight $p_{\max}$ and normalized entropy $H_{\mathrm{norm}}$ on the same sampled rows
    \item Layer-wise gradient norms
\end{itemize}

\subsection{Numerical Checks and Large-$L$ Certificates}

\begin{figure}[t]
\centering
\includegraphics[width=0.85\columnwidth]{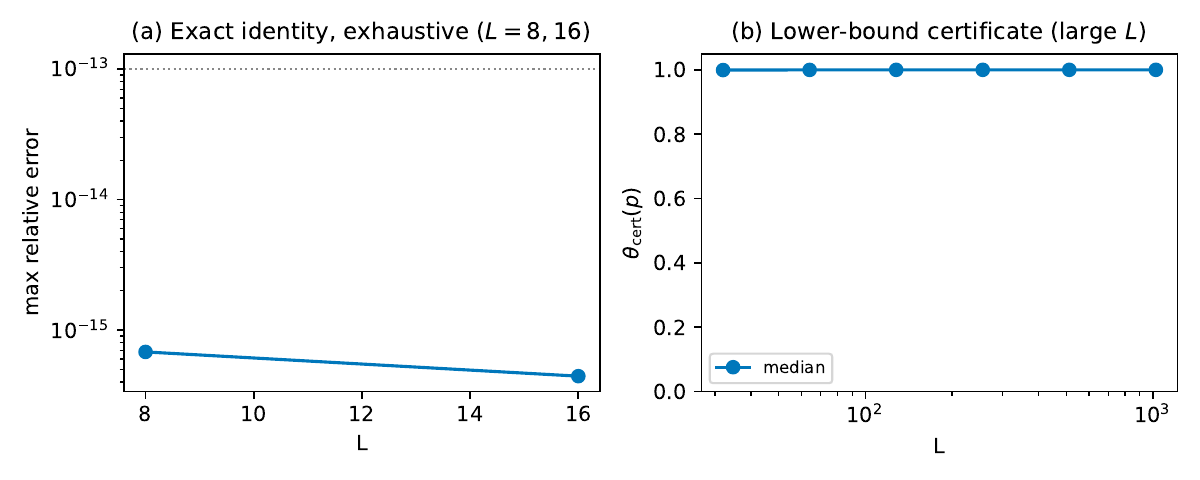}
\caption{Numerical check of Theorem~\ref{thm:softmax}. For $L \in \{8,16\}$, exhaustive enumeration confirms the identity $\|J_\tau(u)\|_{\infty \to 1} = \theta(p)/\tau$ holds to machine precision ($< 10^{-13}$ relative error). For larger $L$, the plotted quantity is the feasible-subset lower-bound certificate $\thetacert(p)$, computed across 500 random distributions per length; it is not an independent verification of the operator norm.}
\label{fig:theorem_verification_app}
\end{figure}

Figure~\ref{fig:theorem_verification_app} provides a numerical consistency check for Theorem~\ref{thm:softmax}. For small sequence lengths $L \in \{8, 16\}$, we perform an exhaustive numerical comparison by enumerating all $2^L$ subsets. For each of 500 random logit vectors $u \sim \mathcal{N}(0, I)$:
\begin{enumerate}[leftmargin=*,itemsep=2pt]
    \item Compute $p = \softmax(u/\tau)$ with $\tau = 1$
    \item Compute the left-hand side: construct the Jacobian $J = \Diag(p) - pp^\top$ and find $\|J\|_{\infty \to 1}$ by exhaustive search over all sign vectors $x \in \{\pm 1\}^L$
    \item Compute the right-hand side: $\theta(p)/\tau$ by exhaustive search over all subsets $S \subseteq [L]$
    \item Record the relative error $|\text{LHS} - \text{RHS}|/\text{RHS}$
\end{enumerate}

For the tested lengths $L\in\{8,16\}$, the maximum relative error across all tests is below $10^{-13}$, consistent with double-precision floating-point arithmetic. This numerical check is consistent with Theorem~\ref{thm:softmax}, whose proof establishes an exact equality rather than merely an upper bound.

For larger $L \in \{32, 64, 128, 256, 512, 1024\}$, exhaustive enumeration is computationally infeasible. Instead, we use a greedy procedure to construct a feasible subset and therefore a lower-bound certificate for $\theta(p)$. This is not an independent verification of the operator norm at large $L$.

\paragraph{Greedy certificate for $\theta(p)$.}
\label{app:theta}
For large $L$, we compute the certificate $\thetacert(p) \le \theta(p) = 4\max_{S \subseteq [L]} p(S)(1-p(S))$ as follows:
\begin{enumerate}[leftmargin=*,itemsep=2pt]
    \item Sort probabilities in descending order: $p_{(1)} \ge p_{(2)} \ge \cdots \ge p_{(L)}$
    \item Compute cumulative sums $c_k = \sum_{i=1}^k p_{(i)}$
    \item Find $k^* = \arg\min_k |c_k - 0.5|$
    \item Return $\thetacert(p) = 4 \cdot c_{k^*} \cdot (1 - c_{k^*})$
\end{enumerate}
The greedy prefix subset is feasible but \emph{not} optimal in general: e.g.\ for $p=(0.4,0.3,0.2,0.1)$ the greedy prefix is $\{1\}$ with $p(S)=0.4$ giving $\thetacert= 0.96$, while the true optimum is $\{1,4\}$ or $\{2,3\}$ with $p(S)=0.5$ giving $\theta=1$. (By Remark~\ref{rem:maxcut}, exact computation of $\theta$ is a \textsc{Partition}-type subset optimization, which is why we report certificates.) Since $\thetacert(p)\le\theta(p)\le 1$, a certificate close to one localizes the exact value: $\theta(p)\in[\thetacert(p),1]$. At $L\in\{8,16\}$, where exhaustive search is available, the greedy certificate is frequently strictly below the exact value on random Gaussian-logit distributions (over $500$ distributions per length, it matches the exact value on $17.8\%$ of instances at $L=8$ and $2.6\%$ at $L=16$, with worst-case gap $\theta-\thetacert\approx 0.11$); this is precisely why we report $\thetacert$ as a lower bound rather than an estimate. Theorem~\ref{thm:softmax} is proved analytically in Appendix~\ref{app:softmax}; the greedy procedure is used only to construct empirical lower bounds at large $L$.

\FloatBarrier
\subsection{Multi-Seed Results}

\begin{table}[t]
\centering
\caption{Early-training comparison across three seeds. Loss is averaged over the final $50$ optimizer updates; $\pm$ is the sample standard deviation across the three seeds (denominator $n-1$). Spike counts cover the logged trajectory and follow the rule defined in Section~\ref{sec:experiments}. Certificate entries are the last-checkpoint means across sampled layers and seeds of the per-layer medians. The table does not represent converged performance.}
\label{tab:multiseed_app}
\begin{tabular}{lccc}
\toprule
\textbf{Architecture} & \textbf{Mean Loss (final 50 updates)} & \textbf{Spike updates} & \textbf{Last-checkpoint $\thetacert/\tau$} \\
\midrule
Pre-LN & $4.509 \pm 0.005$ & $9.0 \pm 1.0$ & $0.99870$ \\
Post-LN & $6.771 \pm 1.551$ & $35.3 \pm 11.7$ & $0.99999$ \\
\bottomrule
\end{tabular}
\end{table}

Table~\ref{tab:multiseed_app} summarizes the three-seed results. Pre-LN losses are tightly clustered ($4.503$, $4.510$, $4.513$), whereas Post-LN outcomes split sharply by seed: one seed reaches $4.98$ and two plateau near $7.67$, so the mean alone obscures the seed-level behavior. Pre-LN has lower loss for all three matched seeds. Post-LN also registers roughly four times as many spike updates under the protocol rule. Its higher median certificate ($0.99999$ versus $0.99870$ for Pre-LN) reflects the near-uniform attention of the two plateaued seeds: their sampled-row maximum weights closely match the uniform-support baseline in these two summary statistics (median $0.0020$ vs.\ $0.0019$; mean $0.0041$ vs.\ $0.0041$). By contrast, Pre-LN develops substantially more concentrated attention while retaining a near-maximal median certificate.

\FloatBarrier
\subsection{Sharpening vs.\ Bisectability Trajectories}
\label{app:theta-traj}

Figure~\ref{fig:theta_trajectory_app} reports the sharpening-vs-certificate trajectories that motivate the body claim of Section~\ref{sec:experiments}: attention concentration changes substantially over the logged trajectory (normalized entropy $0.987\!\to\!0.726$; mean row maximum weight $0.010\!\to\!0.174$), while the median certificate $\thetacert(p)/\tau$ remains near one at every logged checkpoint (last-checkpoint mean of the per-layer medians, $0.9987$). The percentile envelope widens once sharpening begins at the end of warmup, reaching $[0.26, 1.00]$ at the last checkpoint (minimum per-layer 5th percentile to maximum per-layer 95th percentile over layers and seeds; the lower edge is driven by the middle sampled layer).

Because each per-layer median certificate remains near one, at least half of the sampled rows within each layer have exact $\theta$ at least as large as that median. The lower envelope identifies rows with smaller certificates, but a small greedy certificate alone does not upper-bound the exact $\theta$. For rows with $p_{\max}\ge1/2$, Corollary~\ref{cor:regimes} instead gives the exact dominant-atom identity
\[
\theta(p)=4p_{\max}(1-p_{\max}),
\]
which directly characterizes the lower sensitivity of the strongly concentrated rows. The corresponding layerwise breakdown (Figure~\ref{fig:theta_layerwise_app}) localizes the concentrated rows: medians are near one at all three sampled depths, and the widened lower whisker appears at the middle layer.

\begin{figure}[h]
\centering
\includegraphics[width=0.88\columnwidth]{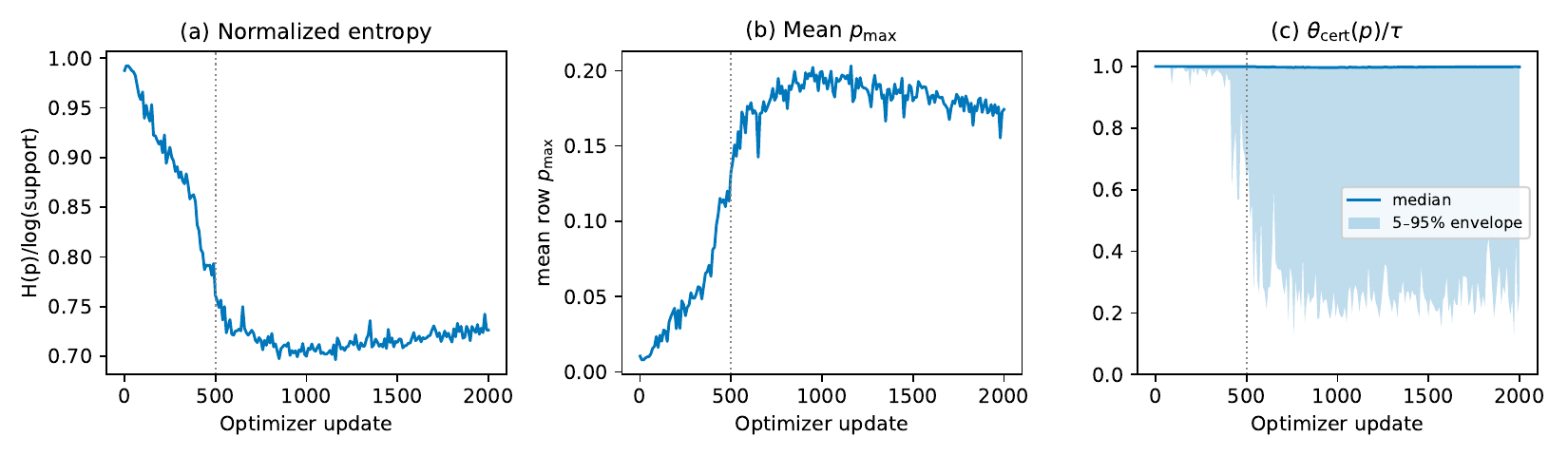}
\caption{Attention concentration changes while the median lower-bound certificate remains high. Across $3$ seeds on $774$M Pre-LN over the logged early-training trajectory: normalized entropy decreases ($0.987\!\to\!0.726$) and mean row maximum weight increases ($0.010\!\to\!0.174$), while the certificate stays near one at every logged checkpoint. The certificate curve is the mean across sampled layers and seeds of the per-layer medians (last checkpoint $0.9987$); the shaded envelope runs from the minimum per-layer 5th percentile to the maximum per-layer 95th percentile across layers and seeds ($[0.26,1.00]$ at the last checkpoint). Entropy and $p_{\max}$ curves are the mean across sampled layers of the per-layer sampled-row mean, averaged over seeds; the dotted vertical line marks the end of warmup at update $500$. The widening envelope shows that a minority of rows develops substantially smaller certificates; the dominant-atom formula in Corollary~\ref{cor:regimes} characterizes exact sensitivity for rows with $p_{\max}\ge1/2$.}
\label{fig:theta_trajectory_app}
\end{figure}

\begin{figure}[h]
\centering
\includegraphics[width=0.85\columnwidth]{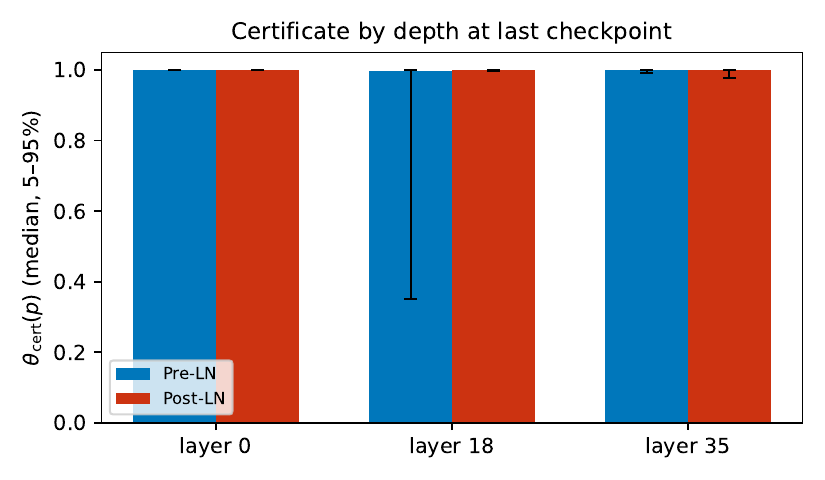}
\caption{Certificate $\thetacert(p)$ at layers $0$, $18$, $35$ at the last logged checkpoint, for a single seed (seed~$123$). Bars show that layer's sampled-row median; whiskers span that layer's sampled-row 5th and 95th percentiles. Medians are near one at all three sampled depths for both architectures, so the near-one median certificate is not confined to the final layer; the Pre-LN middle layer shows the widened lower whisker produced by its concentrated minority of rows (down to $\approx 0.35$). Temporal behavior over the logged trajectory is shown in Figure~\ref{fig:theta_trajectory_app}.}
\label{fig:theta_layerwise_app}
\end{figure}

\FloatBarrier
\section{Extended Related Work}
\label{app:related-extended}

\paragraph{Lipschitz analysis of attention.}
\citet{kim2021lipschitz} established that unconstrained self-attention is not globally Lipschitz, motivating spectral regularization approaches \citep{dasoulas2021lipschitznorm,qi2023lipsformer}. Recent work derives tighter local bounds by incorporating attention-distribution properties: \citet{yudin2025localLipschitz} bound the spectral norm using ordinal statistics of attention probabilities, while \citet{castin2024smooth} study the Lipschitz properties of self-attention and their dependence on sequence length. Separately, \citet{nair2025softmax} proves the tight global constant $1/2$ for softmax uniformly across all same-norm $\ell_p\to\ell_p$ geometries. Our result differs from these lines in both the induced norm and the level of analysis: it is an exact, distribution-dependent characterization of the local $\ell_\infty\to\ell_1$ (cross-norm) Jacobian norm (Theorem~\ref{thm:softmax}), propagated through complete pre-LN and post-LN layers. The same-norm global constants and the cross-norm distribution-dependent identity concern different induced norms and are not directly ordered. Remark~\ref{rem:maxcut} records the classical combinatorial content of the identity (a scaled weighted \textsc{Max-Cut} value of a rank-one-weighted graph); the contribution is its interpretation for attention and its layerwise propagation.

\paragraph{Comparison with Edelman et al.\ (2022) in detail.}
\citet{edelman2022inductive} use a max-row norm as a sub-step inside generalization-bound covering-number arguments. Their Lemma~A.14 bounds an input-Lipschitz constant for a single attention head with fixed weights, combining $W^Q$ and $W^K$ into a single spectral factor $\|W_K W_Q^\top\|$ and not including the output projection $W^O$ as an explicit factor. Our Theorem~\ref{thm:mha} retains the four projections separately and thereby distinguishes the lower-degree value pathway from the softmax-mediated attention-weight pathway of a complete MHA block. We do not infer a universal depth-scaling exponent from this decomposition. The two results have different goals (generalization vs.\ training stability) and different downstream consequences.

\paragraph{Normalization and deep transformers.}
The original transformer \citep{vaswani2017attention} used post-LayerNorm; \citet{xiong2020layer} showed pre-LN dramatically improves stability for deep networks; \citet{liu2020understanding} connected this to gradient flow. We provide a narrower structural result: a pre-LN residual sublayer contains an additive identity term, whereas a post-LN residual sublayer does not (Theorems~\ref{thm:preln-gradient}--\ref{thm:postln-gradient}), and Corollary~\ref{cor:ln-sequence} gives a sufficient condition under which the LayerNorm-only term in the post-LN expansion contracts geometrically at the sequence level. Neither result establishes contraction of the complete post-LN Jacobian.

\paragraph{Residual scaling and initialization.}
DeepNorm \citep{wang2022deepnet} combines a residual multiplier with asymmetric, depth-dependent initialization to stabilize very deep post-LN transformers. Related approaches include Fixup \citep{zhang2019fixup} and T-Fixup \citep{huang2020improving}. We discuss DeepNorm only as related motivation for studying projection-norm pathways; its prescription and exponent are not derived by the present analysis.

\paragraph{Initialization-time properties of attention.}
\citet{noci2022signal,naitsaada2025attention} document that attention distributions are near-uniform at random initialization, which supports our claim that $\theta(p)\approx 1$ at initialization for rows with sufficiently large causal support.

\paragraph{Attention sinks.}
Attention-sink behavior \citep{xiao2024efficient} motivates asking when concentration on one token reduces the exact sensitivity factor. Corollary~\ref{cor:regimes} shows that if a token carries mass at least $1/2$, then $\theta(p)=4p_{\max}(1-p_{\max})$, while in general $\theta(p)\ge 1-p_{\max}^2$. Moderate concentration can therefore coexist with large $\ell_\infty\to\ell_1$ softmax sensitivity; $\theta$ becomes small only as the distribution approaches a dominant one-hot state.

\end{document}